\documentclass{article}

\usepackage[final]{neurips_2024}

\usepackage[utf8]{inputenc} 
\usepackage[T1]{fontenc}    
\usepackage{hyperref}       
\usepackage{url}            
\usepackage{booktabs}       
\usepackage{amsfonts}       
\usepackage{nicefrac}       
\usepackage{microtype}      
\usepackage{xcolor}         
\usepackage{amsthm, amssymb, amsfonts, latexsym, mathtools}
\usepackage{algorithm}
\usepackage[noend]{algpseudocode}
\usepackage{comment}
\usepackage{here}
\usepackage{caption}
\usepackage{subcaption}
\usepackage{wrapfig}
\usepackage{multirow}
\usepackage{color}
\usepackage{wrapfig}
\usepackage{threeparttable}

\newtheorem{theorem}{Theorem}
\newtheorem{proposition}{Proposition}
\newtheorem{lemma}{Lemma}

\newtheorem{assumption}{Assumption}

\usepackage{amsmath,amsfonts,bm}

\def\eqref#1{(\ref{#1})}

\def\1{\bm{1}}

\def\vx{{\bm{x}}}
\def\vy{{\bm{y}}}

\DeclareMathAlphabet{\mathsfit}{\encodingdefault}{\sfdefault}{m}{sl}
\SetMathAlphabet{\mathsfit}{bold}{\encodingdefault}{\sfdefault}{bx}{n}

\DeclareMathOperator*{\argmin}{arg\,min}

\title{Parameter-free Clipped Gradient Descent Meets Polyak}

\author{%
  Yuki Takezawa$^{1,2}$, Han Bao$^{1,2}$, Ryoma Sato$^{3}$, Kenta Niwa$^{4}$, Makoto Yamada$^{2}$ \\
  $^1$Kyoto University, $^2$OIST, $^3$NII, $^4$NTT Communication Science Laboratories
}

\begin{document}

\maketitle

\begin{abstract}
Gradient descent and its variants are de facto standard algorithms for training machine learning models.
As gradient descent is sensitive to its hyperparameters, we need to tune the hyperparameters carefully using a grid search. 
However, the method is time-consuming, particularly when multiple hyperparameters exist.
Therefore, recent studies have analyzed parameter-free methods that adjust the hyperparameters on the fly.
However, the existing work is limited to investigations of parameter-free methods for the stepsize, and parameter-free methods for other hyperparameters have not been explored.
For instance, although the gradient clipping threshold is a crucial hyperparameter in addition to the stepsize for preventing gradient explosion issues,
none of the existing studies have investigated parameter-free methods for clipped gradient descent.
Therefore, in this study, we investigate the parameter-free methods for clipped gradient descent.
Specifically, we propose Inexact Polyak Stepsize, which converges to the optimal solution without any hyperparameters tuning, and its convergence rate is asymptotically independent of $L$ under $L$-smooth and $(L_0, L_1)$-smooth assumptions of the loss function, similar to that of clipped gradient descent with well-tuned hyperparameters. 
We numerically validated our convergence results using a synthetic function and demonstrated the effectiveness of our proposed methods using LSTM, Nano-GPT, and T5.
\end{abstract}

\section{Introduction}
\label{sec:introduction}
We consider the convex optimization problem:
\begin{align}
\label{eq:convex_optimization}
    \min_{\vx \in \mathbb{R}^d} f (\vx),
\end{align}
where the loss function $f$ is convex and lower bounded.
In this setting, gradient descent and its variants \citep{duchi2011adaptive,kingma2015adam} are the de facto standard algorithms to minimize the loss function.
The performance of the algorithm is highly sensitive to the hyperparameter settings, necessitating the careful tuning of the hyperparameters to achieve best performance.
More specifically, when the loss function is $L$-smooth, gradient descent can achieve the optimal convergence rate $\mathcal{O}(\tfrac{L \| \vx_0 - \vx^\star \|^2}{T})$ when we set the stepsize to $\tfrac{1}{L}$ where $\vx_0$ is the initial parameter and $\vx^\star$ is the optimal solution \citep{nesterov2018lectures}.
Unfortunately, parameter $L$ is problem-specific and unavailable in practice.
Thus, gradient descent must be executed in many times with different hyperparameters to identify the good hyperparameter settings,
which is a very time-consuming process. 
Notably, when multiple hyperparameters are under consideration, this hyperparameter search becomes computationally more demanding.

Several recent studies have examined \textit{parameter-free methods} for tuning hyperparameters on the fly \citep{berrada2020training,defazio2023learning,orvieto2022dynamics,jiang2023adaptive,ivgi2023dog,khaled2023dowg,orabona2017training,carmon2022making}.\footnote{We use \textit{parameter-free methods} to describe algorithms that provably converge to the optimal solution without any problem-specific parameters for setting their stepsizes and other hyperparameters.}
These methods automatically adjust the stepsize during the training and are guaranteed to converge to the optimal solution without tuning the stepsize.
In other words, the stepsize did not require tuning using the grid search.
However, the existing parameter-free methods only focus on the stepsize,
and parameter-free methods for other hyperparameters have not been explored.
For example, in addition to the stepsize, the gradient clipping threshold is an important hyperparameter for training language models \citep{pascanu2013on,zhang2020improved,zhang2020why,zhang2020adaptive}.

Clipped gradient descent can achieve the convergence rate $\mathcal{O} (\tfrac{L_0 \| \vx_0 - \vx^\star\|^2}{T})$ under the assumption that the loss function is $(L_0, L_1)$-smooth when we use the optimal stepsize and gradient clipping threshold \citep{koloskova2023revisiting}.
In many cases, $L_0$ is significantly smaller than $L$ \citep{zhang2020why}.
Thus, by comparing with the convergence rate of gradient descent $\mathcal{O}(\tfrac{L \| \vx_0 - \vx^\star \|^2}{T})$, 
gradient clipping often allows gradient descent to converge faster.
However, we must carefully tune two hyperparameters, stepsize and gradient clipping threshold, to achieve this convergence rate.
If the gradient clipping threshold is too large, the gradient clipping fails to accelerate the convergence.
Moreover, if the gradient clipping threshold is too small, gradient clipping deteriorates rather than accelerating the convergence rate.
\textit{Can we develop a parameter-free method whose convergence rate is asymptotically independent of $L$ under $(L_0, L_1)$-smoothness?}

In this study, we investigate a parameter-free method for clipped gradient descent.
First, we provide the better convergence rate of Polyak stepsize \citep{polyak1987introduction} under $(L_0, L_1)$-smoothness.
We discover that the convergence rate of Polyak stepsize matches that of clipped gradient descent with well-tuned stepsize and gradient clipping threshold.
Although the convergence rate of Polyak stepsize is asymptotically independent of $L$ under $(L_0, L_1)$-smooth assumption as clipped gradient descent, it still requires the minimum loss value, which is a problem-specific value.
Thus, we make Polyak stepsize parameter-free without losing this property under $(L_0, L_1)$-smoothness by proposing \textbf{Inexact Polyak Stepsize},
which converges to the optimal solution without any problem-specific parameters.
We numerically evaluated Inexact Polyak Stepsize using a synthetic function and neural networks,
validating our theory and demonstrating the effectiveness of Inexact Polyak Stepsize.

\section{Preliminary}
\subsection{Gradient descent \& $L$-smoothness}
One of the most fundamental algorithms for solving Eq.~\eqref{eq:convex_optimization} represents the gradient descent:
\begin{align*}
    \vx_{t+1} = \vx_{t} - \eta_t \nabla f (\vx_t),
\end{align*}
where $\vx_0 \in \mathbb{R}^d$ is the initial parameter and $\eta_t > 0$ is the stepsize at $t$-th iteration.
To ensure that gradient descent converges to the optimal solution quickly, we must carefully tune the stepsize $\eta_t$.
When the stepsize is too large, the training collapses.
By contrast, when the stepsize is too small, the convergence rate becomes too slow.
Thus, we must search for a proper stepsize as the following theorem indicates.

\begin{assumption}[$L$-smoothness]
\label{assumption:smooth}
There exists a constant $L > 0$ that satisfies the following for all $\vx, \vy \in \mathbb{R}^d$:
\begin{equation}
    \| \nabla f (\vx) - \nabla f (\vy) \| \leq L \| \vx - \vy \|.
\end{equation}
\end{assumption}

\begin{theorem}[{\citet[Corollary 2.1.2]{nesterov2018lectures}}]
\label{theorem:gradient_descent}
Assume that $f$ is convex and $L$-smooth, and there exists an optimal solution $\vx^\star \coloneqq \argmin_{\vx \in \mathbb{R}^d} f (\vx)$.
Then, gradient descent with stepsize $\eta_t = \tfrac{1}{L}$ satisfies
\begin{equation}
    f (\bar{\vx}) - f (\vx^\star) \leq \mathcal{O} \left( \frac{L \| \vx_0 - \vx^\star \|^2}{T} \right),
\end{equation}
where $\bar{\vx} \coloneqq \tfrac{1}{T}\sum_{t=0}^{T-1} \vx_t$ and $T$ is the number of iterations.
\end{theorem}

\subsection{Clipped gradient descent \& $(L_0, L_1)$-smoothness}
\label{sec:clipped_gradient_descent}
Gradient clipping is widely used to stabilize and accelerate the training of gradient descent \citep{pascanu2013on,devlin2019bert}.
Let $c > 0$ be the threshold for gradient clipping. Clipped gradient descent is given by:
\begin{align}
    \vx_{t+1} = \vx_t - \eta_t \min \left\{1, \frac{c}{\| \nabla f (\vx_t)\|} \right\} \nabla f(\vx_t).
\end{align}
Many prior studies investigated the theoretical benefits of gradient clipping \citep{koloskova2023revisiting,zhang2020improved,zhang2020why,zhang2020adaptive,li2022high,sadiev2023high}.
\citet{zhang2020why} experimentally found that the gradient Lipschitz constant decreases during the training of various neural networks and is highly correlated with gradient norm $\| \nabla f (\vx) \|$.
To describe this phenomenon, \citet{zhang2020improved} introduced a novel smoothness assumption called $(L_0, L_1)$-smoothness. Then, it has been experimentally demonstrated that the local gradient Lipschitz constant $L_0$ is thousands of times smaller than the global gradient Lipschitz constant $L$.

\begin{assumption}[$(L_0,L_1)$-smoothness]
\label{assumption:generalized_smooth}
There exists constants $L_0 > 0$ and $L_1 > 0$ that satisfy the following for all $\vx, \vy \in \mathbb{R}^d$ with $\| \vx - \vy \| \leq \frac{1}{L_1}$:
\begin{equation}
    \| \nabla f (\vx) - \nabla f (\vy) \| \leq (L_0 + L_1 \| \nabla f (\vx) \|) \| \vx - \vy \|.
\end{equation}
\end{assumption}

Note that $(L_0, L_1)$-smoothness is strictly weaker than $L$-smoothness because $(L_0, L_1)$-smoothness covers $L$-smoothness by taking $L_1=0$.
Using the $(L_0, L_1)$-smoothness assumption, the convergence rate of clipped gradient descent was established as follows.
\begin{theorem}[{\citet[Theorem 2.3]{koloskova2023revisiting}}]
\label{theorem:clip}
Assume that $f$ is convex, $L$-smooth, and $(L_0, L_1)$-smooth, and there exists an optimal solution $\vx^\star \coloneqq \argmin_{\vx \in \mathbb{R}^d} f (\vx)$.
Then, clipped gradient descent with $\eta_t = \tfrac{1}{L_0}$ and $c = \tfrac{L_0}{L_1}$ satisfies:
\begin{equation}
    f(\bar{\vx}) - f (\vx^\star) \leq \mathcal{O} \left( \frac{ L_0 \| \vx_0 - \vx^\star\|^2}{T} + \frac{L L_1^2 \| \vx_0 - \vx^\star \|^4}{T^2}\right),
\end{equation}
where $\bar{\vx} \coloneqq \tfrac{1}{T}\sum_{t=0}^{T-1} \vx_t$ and $T$ is the number of iterations.
\end{theorem}

When the number of iterations $T$ is large, the first term $\mathcal{O} (\tfrac{L_0 \| \vx_0 - \vx^\star \|^2}{T})$ becomes dominant, and the convergence rate of clipped gradient descent is asymptotically independent of $L$.
Gradient clipping allows for the use of a larger stepsize, and thus, gradient descent converges faster because of $L_0 \ll L$.
We can interpret $L \simeq L_0 + L_1 \sup_\vx \| \nabla f (\vx) \|$.
The stepsize of gradient descent in Theorem \ref{theorem:gradient_descent} is $\tfrac{1}{L_0 + L_1 \sup_\vx \| \nabla f (\vx) \|}$, which is typically very small.
By comparing with gradient descent,
the coefficient multiplied by the gradient of clipped gradient descent in Theorem \ref{theorem:clip} is $\min \{ \tfrac{1}{L_0}, \tfrac{1}{L_1 \| \nabla f (\vx_t)\|}\}$,
which is larger than $\tfrac{1}{L_0 + L_1 \sup_\vx \| \nabla f (\vx) \|}$.
Specifically, when parameter $\vx$ is close to the optimal solution $\vx^\star$ (i.e., $\| \nabla f (\vx) \|$ is small),
clipped gradient descent can use a larger stepsize and then reach the optimal solution faster than gradient descent.

\subsection{Polyak stepsize}

When $f$ is convex, $\vx_{t+1}$ and $\vx_t$ generated by gradient descent satisfy $\| \vx_{t+1} - \vx^\star \|^2 \leq \| \vx_t - \vx^\star \|^2 - 2 \eta_t (f (\vx_t) - f (\vx^\star)) + \eta_t^2 \| \nabla f (\vx_t) \|^2$.
By minimizing the right-hand side, we can derive well-known Polyak stepsize \citep{polyak1987introduction}:
\begin{align}
\label{eq:polyak}
    \eta_t = \frac{f (\vx_t) - f^\star}{\| \nabla f (\vx_t)\|^2},
\end{align}
where $f^\star \coloneqq f (\vx^\star)$.
When $f$ is $L$-smooth, gradient descent with Polyak stepsize converges to the optimal solution as quickly as gradient descent with $\eta_t = \tfrac{1}{L}$.

\begin{theorem}[{\citet[Theorem 1]{hazan2019revisiting}}]
\label{theorem:gd}
Assume that $f$ is convex and $L$-smooth, and there exists an optimal solution $\vx^\star \coloneqq \argmin_{\vx \in \mathbb{R}^d} f (\vx)$. 
Then, gradient descent with Polyak stepsize Eq.~\eqref{eq:polyak} satisfies:
\begin{equation}
    f (\bar{\vx}) - f (\vx^\star) \leq \mathcal{O} \left( \frac{L \| \vx_0 - \vx^\star \|^2}{T} \right),
\end{equation}
where $\bar{\vx} \coloneqq \tfrac{1}{T}\sum_{t=0}^{T-1} \vx_t$ and $T$ is the number of iterations.
\end{theorem}

In addition to the $L$-smooth setting, Polyak stepsize is known to cause gradient descent to converge to the optimal solution with the optimal rate among various settings, e.g., non-smooth convex, smooth convex, and strongly convex settings \citep{hazan2019revisiting}.

\section{Improved convergence result of Polyak stepsize}
\label{sec:improved_convergence_result}

Before proposing a parameter-free method for clipped gradient descent, 
in this section, we present a new convergence analysis of Polyak stepsize under $(L_0, L_1)$-smoothness.
Surprisingly, our new analysis reveals that Polyak stepsize achieves exactly the same convergence rate as \textit{clipped} gradient descent with appropriate hyperparameters.
A bunch of prior studies established the convergence rates of Polyak stepsize,
and it is well-known that Polyak stepsize allows gradient descent to converge as fast as the optimal stepsize.
However, our theorem finds that Polyak stepsize achieves a faster convergence rate than gradient descent with the optimal stepsize as clipped gradient descent, and none of the existing studies have found this favorable property of Polyak stepsize.

\subsection{Connection between Polyak stepsize and clipped gradient descent}
\label{sec:relationship}

Under $(L_0, L_1)$-smoothness, we can obtain the following results.
\begin{proposition}
\label{proposition:connection}
Assume that $f$ is convex and $(L_0, L_1)$-smooth. 
Then, Polyak stepsize Eq.~\eqref{eq:polyak} satisfies:
\begin{equation}
    \min\left\{ \frac{1}{4 L_0}, \frac{1}{4 L_1 \| \nabla f (\vx_t) \|} \right\} \leq \frac{f (\vx_t) - f^\star}{\| \nabla f (\vx_t)\|^2}.
\end{equation}
\end{proposition}
\begin{proof}
Assumption \ref{assumption:generalized_smooth} and Lemma \ref{lemma:generalized_smooth} imply 
\begin{align*}
    \frac{f (\vx_t) - f^\star}{\| \nabla f (\vx_t)\|^2} \geq \frac{1}{2 (L_0 + L_1 \| \nabla f (\vx_t) \|)}.
\end{align*}
When $\| \nabla f (\vx_t) \| < \tfrac{L_0}{L_1}$, Polyak stepsize is bounded from below by $\tfrac{1}{4 L_0}$.
When $\| \nabla f (\vx_t) \| \geq \tfrac{L_0}{L_1}$, we have
\begin{align*}
    \frac{1}{2 (L_0 + L_1 \| \nabla f (\vx_t) \|)}
    \geq \frac{1}{4 L_1 \| \nabla f (\vx_t) \|}.
\end{align*}
Therefore, we can conclude the statement.
\end{proof}

Under $L$-smoothness, the lower bound of Polyak stepsize was obtained as follows.
\begin{proposition}[{\citet[Lemma 15]{jiang2023adaptive}}]
\label{proposition:lower_bound}
Assume that $f$ is convex and $L$-smooth.
Then, Polyak stepsize Eq.~\eqref{eq:polyak} satisfies:
\begin{equation}
    \frac{1}{2 L} \leq \frac{f (\vx_t) - f^\star}{\| \nabla f (\vx_t)\|^2}.    
\end{equation}
\end{proposition}

By comparing Propositions \ref{proposition:connection} and \ref{proposition:lower_bound},
Proposition \ref{proposition:connection} shows that Polyak stepsize does not become excessively small when the parameter approaches the optimal solution (i.e., $\| \nabla f (\vx)\|$ approaches zero), similar to clipped gradient descent.
If we choose the stepsize and gradient clipping threshold as in Theorem \ref{theorem:clip}, clipped gradient descent can be written as follows:
\begin{align}
    \vx_{t+1} = \vx_t - \min \left\{\frac{1}{L_0}, \frac{1}{L_1 \| \nabla f (\vx_t)\|} \right\} \nabla f(\vx_t).
\end{align}
Thus, Proposition \ref{proposition:connection} implies that Polyak stepsize can be regarded as internally estimating the hyperparameters for clipped gradient descent, as shown in Theorem \ref{theorem:clip}.

\subsection{Convergence analysis of Polyak stepsize under $(L_0, L_1)$-smoothness}
Based on the relationship between Polyak stepsize and clipped gradient descent in Sec.~\ref{sec:relationship}, we provide a new convergence result for Polyak stepsize under $(L_0, L_1)$-smoothness.
The proof is deferred to Sec.~\ref{sec:proof_of_main}.
\begin{theorem}
\label{theorem:main}
Assume that $f$ is convex, $L$-smooth, and $(L_0, L_1)$-smooth, and there exists an optimal solution $\vx^\star \coloneqq \argmin_{\vx \in \mathbb{R}^d} f (\vx)$.
Let $T$ be the number of iterations and define $\tau \coloneqq \argmin_{0\leq t \leq T-1} f(\vx_t)$. Then, gradient descent with Polyak stepsize Eq.~\eqref{eq:polyak} satisfies:
\begin{equation}
    f(\vx_{\tau}) - f (\vx^\star) \leq \mathcal{O} \left( \frac{ L_0 \| \vx_0 - \vx^\star\|^2}{T} + \frac{L L_1^2 \| \vx_0 - \vx^\star \|^4}{T^2}\right).
\end{equation}
\end{theorem}

By comparing Theorem \ref{theorem:main} with Theorem \ref{theorem:clip}, 
the convergence rate of Polyak stepsize is the same as that of clipped gradient descent.
Thus, Polyak stepsize can converge faster than the optimal stepsize given in Theorem \ref{theorem:gradient_descent} when $L_0 \ll L$.
Many prior studies analyzed the convergence rate of Polyak stepsize and discussed the relationship between Polyak stepsize and gradient descent with the optimal stepsize \citep{polyak1987introduction,loizou2021stochastic,galli2023donot,berrada2020training}.
However, they only recognized Polyak stepsize as making gradient descent converge with the same convergence rate as the optimal stepsize, and none of the prior studies have found this relationship between Polyak stepsize and clipped gradient descent.
Our new convergence result is the first to discover that the Polyak stepsize can achieve the same convergence rate not only as gradient descent with an appropriate stepsize but also as clipped gradient descent with an appropriate stepsize and gradient clipping threshold.

\section{Making clipped gradient descent parameter-free}
\begin{algorithm}[b!]
\caption{Inexact Polyak Stepsize}
\label{alg:proposed_method}
\begin{algorithmic}[1]
\State \textbf{Input:} The number of iterations $T$ and lower bound $l^\star$.
\State $f^\text{best}, \vx^\text{best} \leftarrow f (\vx_0), \vx_0$.
\For{$t = 0, 1, \cdots, T-1$}
\State $\vx_{t+1} \leftarrow \vx_t - \tfrac{f(\vx_t) - l^\star}{\sqrt{T} \| \nabla f (\vx_t) \|^2} \nabla f (\vx_t)$.
\If{$f(\vx_{t+1}) \leq f^\text{best}$}
\State $f^\text{best}, \vx^\text{best} \leftarrow f (\vx_{t+1}), \vx_{t+1}$.\EndIf
\EndFor
\State \Return $\vx^\text{best}$.
\end{algorithmic}
\end{algorithm}

In the previous section, we found that the convergence rate of Polyak stepsize is asymptotically independent of $L$ under $(L_0, L_1)$-smoothness as clipped gradient descent with appropriate hyperparameters.
However, Polyak stepsize requires the minimum loss value $f^\star$, which is a problem-specific parameter.
In this section, we propose a method that can remove the prior knowledge of $f^\star$ from Polyak stepsize without losing the property of asymptotic independence of $L$ under $(L_0, L_1)$-smoothness.

\subsection{Inexact Polyak Stepsize}
To make Polyak stepsize parameter-free, several prior studies have proposed the use of lower bound of $f^\star$ instead of $f^\star$ \citep{loizou2021stochastic,orvieto2022dynamics,jiang2023adaptive}.
The loss functions commonly used in machine learning models are non-negative.
Thus, the lower bound of $f^\star$ is trivially obtained as zero and is not a problem-specific parameter.
By utilizing this lower bound, a straightforward approach to make Polyak stepsize independent of problem-specific parameters is replacing $f^\star$ in Polyak stepsize with the lower bound $l^\star$ as follows:
\begin{align}
\label{eq:straightforward}
    \eta_t = \frac{f(\vx_t) - l^\star}{\| \nabla f (\vx_t) \|^2}.
\end{align}
However, the stepsize in Eq.~\eqref{eq:straightforward} becomes excessively large as the parameter approaches the optimal solution,
and it does not lead to the optimal solution \citep{loizou2021stochastic}.
This is because $\| \nabla f (\vx_t) \|$ approaches zero, while $f (\vx_t) - l^\star$ approaches $f^\star - l^\star (> 0)$,
which makes the stepsize in Eq.~\eqref{eq:straightforward} excessively large as the parameter approaches the optimal solution.
To mitigate this issue, DecSPS \citep{orvieto2022dynamics} and AdaSPS \citep{jiang2023adaptive}, which are parameter-free methods based on Polyak stepsize that use $l^\star$ instead of $f^\star$, make the stepsize monotonically non-increasing to converge to the optimal solution.

However, making the stepsize monotonically non-increasing loses the fruitful property that the convergence rate of Polyak stepsize is asymptotically independent of $L$ as clipping gradient descent under $(L_0, L_1)$-smoothness.
This is because Polyak stepsize and clipped gradient descent make the convergence rate asymptotically independent of $L$ by increasing the stepsize when the parameter approaches the optimal solution.
In fact, we evaluated DecSPS and AdaSPS with a synthetic function in Sec.~\ref{sec:synthetic},
demonstrating that the convergence deteriorates as $L$ increases.

To address this issue, we propose \textbf{Inexact Polyak Stepsize}, whose details are described in Alg.~\ref{alg:proposed_method}.
As discussed above, we cannot make the stepsize decrease to maintain the asymptotic independence of $L$ under $(L_0, L_1)$-smoothness.
Thus, we set the stepsize as follows:
\begin{equation}
\label{eq:inexact_polyak}
    \eta_t = \frac{f(\vx_t) - l^\star}{\sqrt{T} \| \nabla f (\vx_t) \|^2},
\end{equation}
where $T$ denotes the number of iterations.
Instead of making the stepsize decrease, we propose returning the parameter for which the lowest loss is achieved as the final parameter.

\subsection{Convergence analysis of Inexact Polyak Stepsize}
\begin{table}[!b]
\vskip - 0.2 in
\centering
\caption{Summary of convergence rates of parameter-free methods based on Polyak stepsize. All convergence results are the ones under convex, $L$-smoothness, and $(L_0, L_1)$-smoothness. We define $D_T \coloneqq \max_{0 \leq t \leq T} \| \vx_t - \vx^\star\|$.}
\label{table:convergence_rate}
\vskip 0.05 in
\resizebox{\linewidth}{!}{
\begin{tabular}{lcc}
\toprule
Algorithm & Convergence Rate & Assumption \\
\midrule
DecSPS \citep{orvieto2022dynamics}$^{(a)}$ & $\mathcal{O} \left( \tfrac{\max \{ L, \eta_0^{-1} \}D_T^2 + \sigma^2 }{\sqrt{T}} \right)$ & \ref{assumption:smooth}$^{\quad}$ \\
AdaSPS \citep{jiang2023adaptive}$^{(a)}$ & $\mathcal{O} \left( \tfrac{L D_T^2 \sigma}{\sqrt{T}} + \tfrac{L^2 D_T^4}{T} \right)$ & \ref{assumption:smooth}$^{\quad}$ \\
Inexact Polyak Stepsize (This work) &   $\mathcal{O} \left( \frac{ L_0 \| \vx_0 - \vx^\star\|^2 + \sigma^2}{\sqrt{T}} + \frac{L L_1^2 \| \vx_0 - \vx^\star \|^4}{T} + \frac{L_1^2 L \sigma^4}{L_0^2 T} \right)$ & \ref{assumption:smooth}, \ref{assumption:generalized_smooth}$^{(b)}$ \\
\bottomrule
\end{tabular}}
  \begin{tablenotes}
      {\scriptsize
        \item (a) We present the convergence rates of DecSPS and AdaSPS in the deterministic setting to compare DecSPS, AdaSPS, and Inexact Polyak Stepsize in the same deterministic setting, while \citet{orvieto2022dynamics} and \citet{jiang2023adaptive} also analyzed the rate rates in the stochastic setting.
        \item (b) If $f$ is $L$-smooth, $f$ is $(L_0, L_1)$-smooth because $(L_0, L_1)$-smoothness assumption is strictly weaker than $L$-smoothness assumption.
    }
    \end{tablenotes}
\vskip - 0.2 in
\end{table}
The following theorem provides the convergence rate of Inexact Polyak Stepsize.
The proof is deferred to Sec.~\ref{sec:proof_of_lower_bound}.
\begin{theorem}
\label{theorem:main_lower_bound}
Assume that $f$ is convex, $L$-smooth, and $(L_0, L_1)$-smooth, and there exists an optimal solution $\vx^\star \coloneqq \argmin_{\vx\in\mathbb{R}^d} f (\vx)$.
Let $T$ be the number of iterations and $\sigma^2 \coloneqq f^\star - l^\star$.
Then, $\vx$ generated by Alg.~\ref{alg:proposed_method} satisfies:
\begin{equation}
\label{eq:rate_of_inexact_polyak}
    f(\vx) - f (\vx^\star) \leq \mathcal{O} \left( \frac{ L_0 \| \vx_0 - \vx^\star\|^2 + \sigma^2}{\sqrt{T}} + \frac{L L_1^2 \| \vx_0 - \vx^\star \|^4}{T} + \frac{L_1^2 L \sigma^4}{L_0^2 T} \right).
\end{equation}
\end{theorem}

\paragraph{Asymptotic independence of $L$:}
When the number of iterations $T$ is large, only the first term $\mathcal{O} (\tfrac{L_0 \| \vx_0 - \vx^\star\|^2 + \sigma^2}{\sqrt{T}})$ becomes dominant in the convergence rate, which does not depend on $L$.
Thus, Theorem \ref{theorem:main_lower_bound} shows that Inexact Polyak Stepsize successfully inherits the favorable property of Polyak stepsize under $(L_0, L_1)$-smoothness.
In addition to Inexact Polyak Stepsieze, DecSPS \citep{orvieto2022dynamics} and AdaSPS \citep{jiang2023adaptive} have been proposed as parameter-free methods that use $l^\star$ instead of $f^\star$ in Polyak stepsize.
However, these prior methods fail to inherit the favorable property of Polyak stepsize, and their convergence rates deteriorate when $L$ is large because these methods decrease the stepsize during the training.
In fact, we evaluated DecSPS and AdaSPS with a synthetic function in Sec.~\ref{sec:synthetic}, 
demonstrating that convergence rates of DecSPS and AdaSPS are degraded when $L$ becomes large, whereas the convergence rate of Inexact Polyak Stepsize does not depend on $L$.

\paragraph{Removing dependence on $D_T$:}
The convergence rates of DecSPS and AdaSPS depend on $D_T (\coloneqq \max_{0 \leq t \leq T} \| \vx_t - \vx^\star\|)$.
Thus, strictly speaking, these convergence rates cannot show that DecSPS and AdaSPS converge to the optimal solution because $D_T$ may increase as the number of iterations $T$ increases.
For instance, if $D_T$ increase with $\Omega(T^{\frac{1}{4}})$, the convergence rate of AdaSPS is $\mathcal{O}(L \sigma + L^2)$,
which does not show that AdaSPS converges to the optimal solution.
In contrast, the convergence rate in Eq.~\eqref{eq:rate_of_inexact_polyak} depends on only $\| \vx_0 - \vx^\star\|$.
Theorem \ref{theorem:main_lower_bound} indicates that Inexact Polyak Stepsize converges to the optimal solution.

\paragraph{Convergence rate with respect to $T$:}
Inexact Polyak Stepsize successfully achieves the asymptotic independence of $L$, while it slows down the convergence rate with respect to the number of iterations $T$ by comparing clipped gradient descent with proper hyperparameters.
The convergence rate of Inexact Polyak Stepsize $\mathcal{O}(\tfrac{L_0}{\sqrt{T}})$ is not optimal in terms of $T$, and there may be room to improve this rate. 
For instance, the adaptive methods proposed by \citet{hazan2019revisiting} might be used to alleviate this issue.
However, the parameter-free methods for clipped gradient descent have not been explored well in the existing studies. 
We believe that Inexact Polyak Stepsize is the important first step for developing parameter-free clipped gradient descent.

\section{Related work}

\paragraph{Gradient clipping:}
Gradient clipping was initially proposed to mitigate the gradient explosion problem for training RNN and LSTM \citep{mikolov2010recurrent,merity2018regularizing} and is now widely used to accelerate and stabilize the training not only for RNN and LSTM, but also for various machine learning models, especially language models \citep{devlin2019bert,raffel2019exploring}.
Recently, many studies have investigated the theoretical benefits of gradient clipping and analyzed the convergence rate of clipped gradient descent under (1) $(L_0, L_1)$-smoothness assumption \citep{koloskova2023revisiting,zhang2020improved,zhang2020why} and (2) heavy-tailed noise assumption \citep{zhang2020adaptive,li2022high,sadiev2023high}.
(1) \citet{zhang2020why} found that the local gradient Lipschitz constant is correlated with the gradient norm. To describe this phenomenon, \citet{zhang2020why}, \citet{zhang2020improved}, and \citet{koloskova2023revisiting} introduced the new assumption, $(L_0, L_1)$-smoothness,
providing the convergence rate of clipped gradient descent under $(L_0, L_1)$-smoothness.
Then, they showed that gradient clipping can improve the convergence rate of gradient descent,
as we introduced in Sec.~\ref{sec:clipped_gradient_descent}.
(2) Besides $(L_0, L_1)$-smoothness, \citet{zhang2020adaptive} pointed out that the distribution of stochastic gradient noise is heavy-tailed for language models.
Then, it has been shown that gradient clipping can make the stochastic gradient descent robust against the heavy-tailed noise of stochastic gradient \citep{li2022high,sadiev2023high,zhang2020adaptive}.

\paragraph{Parameter-free methods:}
Hyperparameter-tuning is one of the most time-consuming tasks for training machine learning models.
To alleviate this issue, many parameter-free methods that adjust the stepsize on the fly have been proposed, e.g.,
Polyak-based stepsize \citep{berrada2020training,hazan2019revisiting,loizou2021stochastic,mukherjee2023locally,orvieto2022dynamics,jiang2023adaptive}, AdaGrad-based methods \citep{ivgi2023dog,khaled2023dowg}, and Dual Averaging-based methods \citep{orabona2017training,defazio2023learning}.
However, parameter-free methods for hyperparameters, except for stepsizes, have not been studied.
In this work, we studied the parameter-free methods for two hyperparameters, the stepsize and gradient clipping threshold,
and then proposed Inexact Polyak Stepsize, which converges to the optimal solution without tuning any hyperparameters and its convergence rate is asymptotically independent of $L$ as clipped gradient descent with well-tuned hyperparameters.

\section{Numerical evaluation}
In this section, we evaluate our theory numerically.
In Sec.~\ref{sec:synthetic}, we evaluate Polyak stepsize and Inexact Polyak Stepsize using a synthetic function,
varying that their convergence rates are asymptotically independent of $L$.
In Sec.~\ref{section:neural_networks}, we show the results obtained using neural networks.
\begin{figure}[t!]
\centering
\vskip -0.2 in
\begin{subfigure}[b]{0.325\textwidth}
    \centering
    \includegraphics[width=\textwidth]{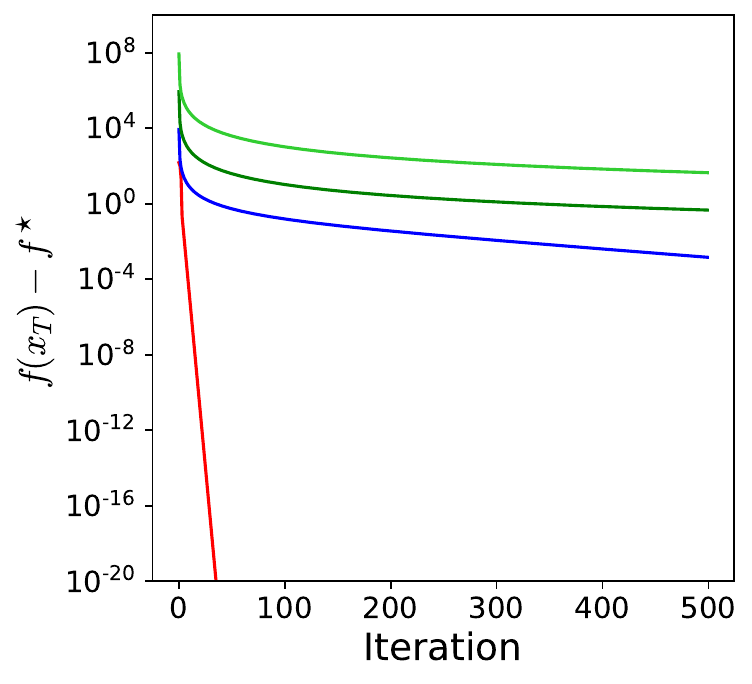}
    \vskip - 0.1 in
    \caption{Gradient Descent}
\end{subfigure}
\hfill
\begin{subfigure}[b]{0.325\textwidth}
    \centering
    \includegraphics[width=\textwidth]{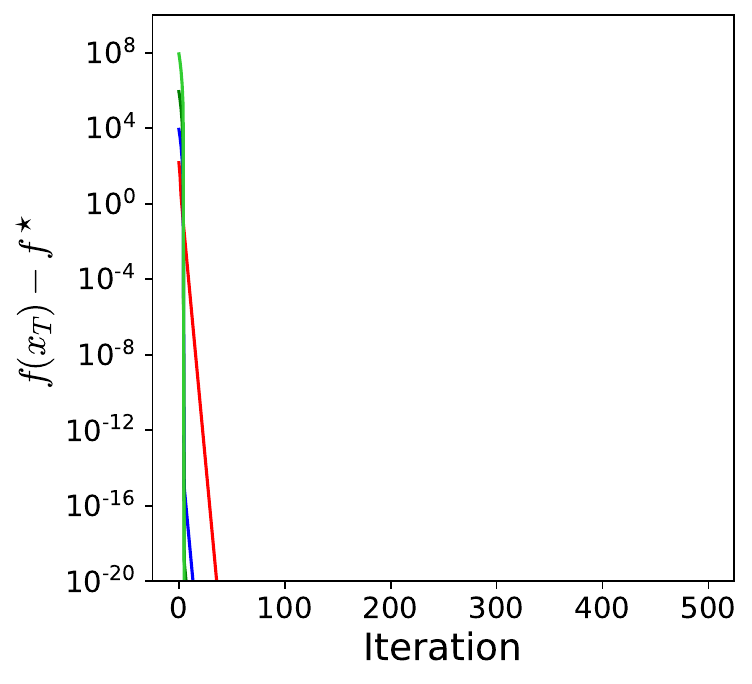}
    \vskip - 0.1 in
    \caption{Clipped Gradient Descent}
\end{subfigure}
\begin{subfigure}[b]{0.325\textwidth}
    \centering
    \includegraphics[width=\textwidth]{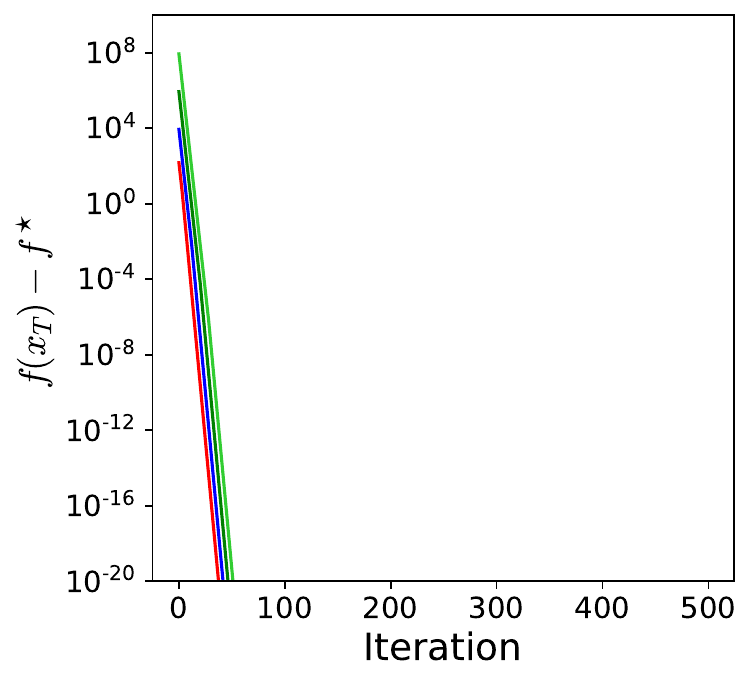}
    \vskip - 0.1 in
    \caption{Polyak Stepsize}
\end{subfigure}
\begin{subfigure}[b]{0.325\textwidth}
    \centering
    \includegraphics[width=\textwidth]{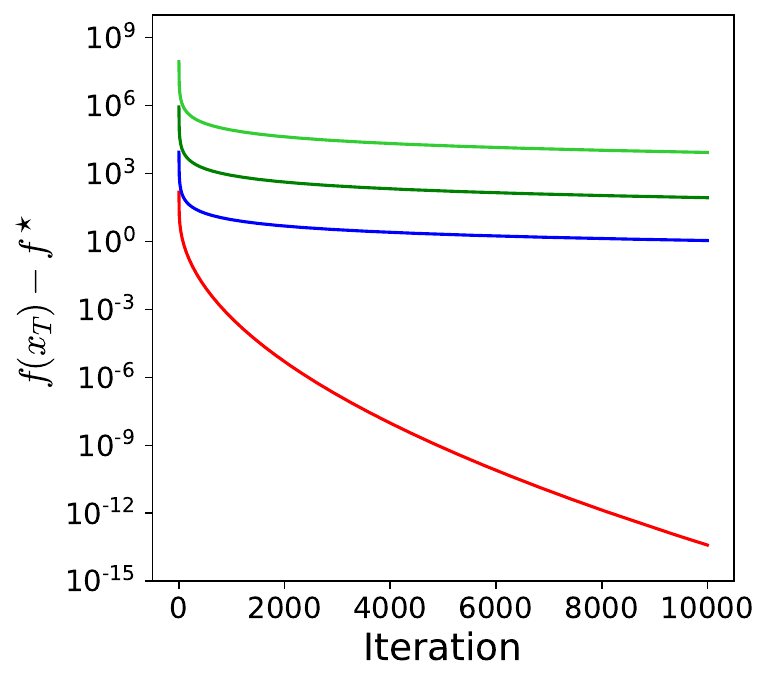}
    \vskip - 0.1 in
    \caption{DecSPS \citep{orvieto2022dynamics}}
\end{subfigure}
\hfill
\begin{subfigure}[b]{0.325\textwidth}
    \centering
    \includegraphics[width=\textwidth]{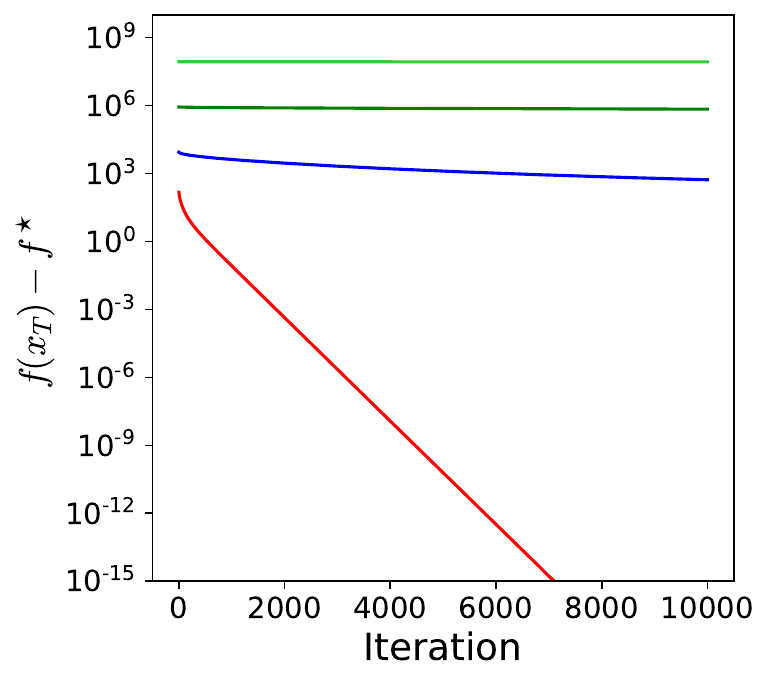}
    \vskip - 0.1 in
    \caption{AdaSPS \citep{jiang2023adaptive}}
\end{subfigure}
\hfill
\begin{subfigure}[b]{0.325\textwidth}
    \centering
    \includegraphics[width=\textwidth]{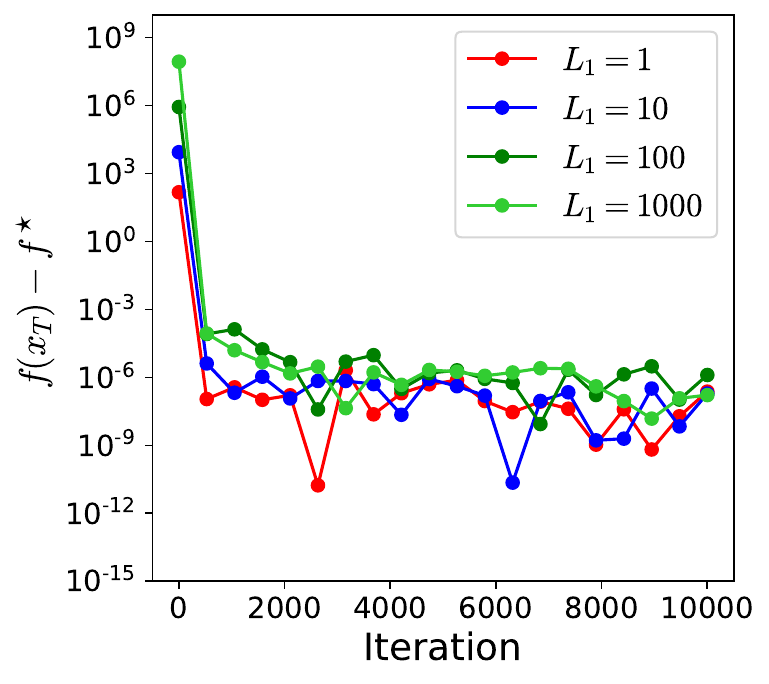}
    \vskip - 0.1 in
    \caption{Inexact Polyak Stepsize}
\end{subfigure}
\caption{Convergence behaviors of various methods with the synthetic function.}
\label{fig:synthetic_with_exact_minimum}      
\end{figure}

\subsection{Synthetic function}
\label{sec:synthetic}
\paragraph{Setting:}
In this section, we validate our theory for Polyak stepsize and Inexact Polyak Stepsize using a synthetic function.
We set the loss function as $f(x) = \tfrac{L_0 L_1^2}{72} x^4 + \tfrac{L_0}{4} x^2 + f^\star$,
which is $(L_0, L_1)$-smooth for any $L_0 > 0$ and $L_1 > 0$ (See Proposition \ref{prop:generalized_smooth} in Appendix).
We set $L_0$ to $1$, $\vx_0$ to $5$, $f^\star = 1$, and $l^\star = 0$ and then evaluated various methods when varying $L_1$.

\paragraph{Results:}
We show the results in Fig.~\ref{fig:synthetic_with_exact_minimum}.
The results indicate that gradient descent converges slowly when $L_1$ is large, whereas Polyak stepsize and clipped gradient descent does not depend on $L_1$.
These observations are consistent with those discussed in Sec.~\ref{sec:improved_convergence_result}, 
which shows that the convergence rate of Polyak stepsize is asymptotically independent of $L$ as in clipped gradient descent.
By comparing DecSPS, AdaSPS, and Inexact Polyak Stepsize, which are parameter-free methods,
the convergence rates of DecSPS and AdaSPS degrade as $L_1$ increases.
Thus, DecSPS and AdaSPS lose the favorable property of asymptotic independence of $L$ under $(L_0, L_1)$-smoothness.
In contrast, the convergence behavior of Inexact Polyak Stepsize does not depend on $L_1$,
which is consistent with Theorem \ref{theorem:main_lower_bound},
and Inexact Polyak Stepsize successfully inherits the Polyak stepsize under $(L_0, L_1)$-smoothness.
\begin{figure}[t!]
\centering
\hfill
\centering
\vskip - 0.1 in
\begin{subfigure}[b]{0.265\textwidth}
\includegraphics[width=\textwidth]{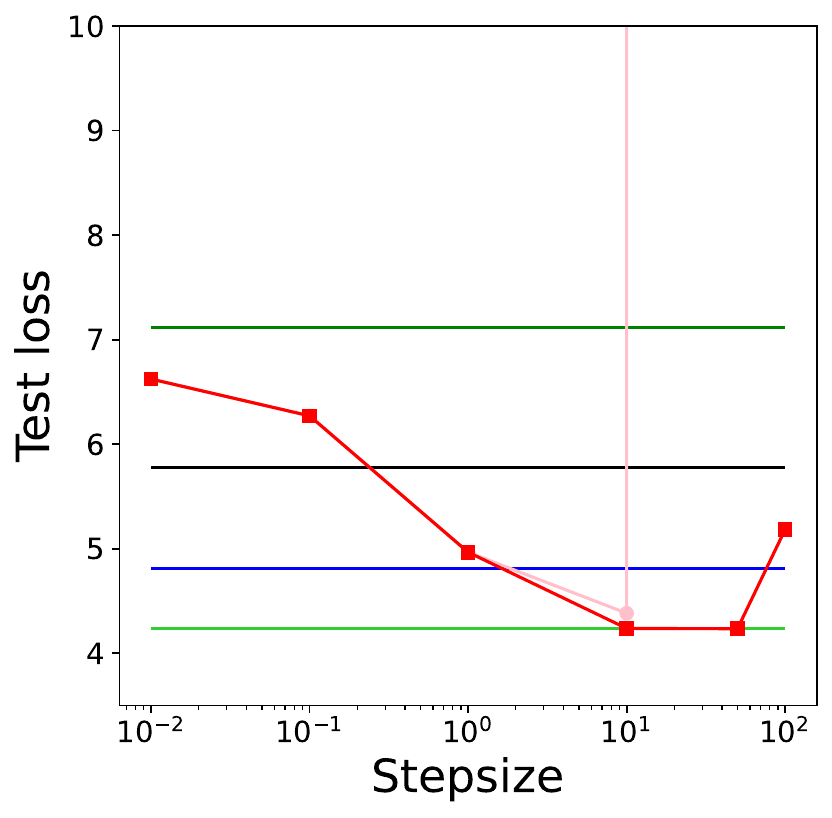}
\caption{LSTM}
\end{subfigure}
\begin{subfigure}[b]{0.265\textwidth}
\includegraphics[width=\textwidth]{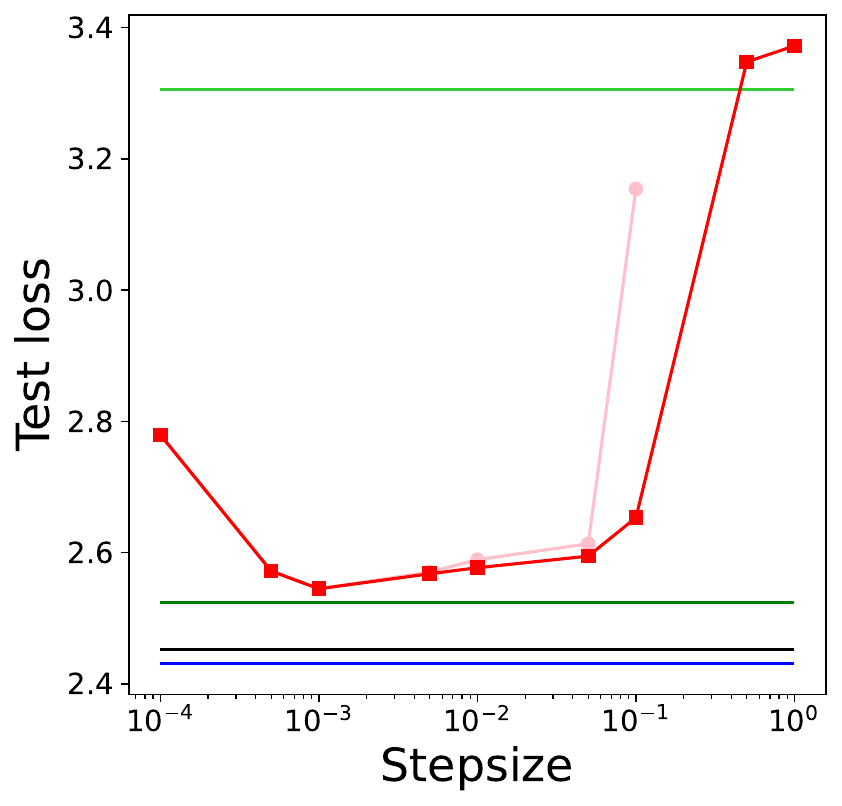}
\caption{Nano-GPT}
\end{subfigure}
\begin{subfigure}[b]{0.455\textwidth}
\includegraphics[width=\textwidth]{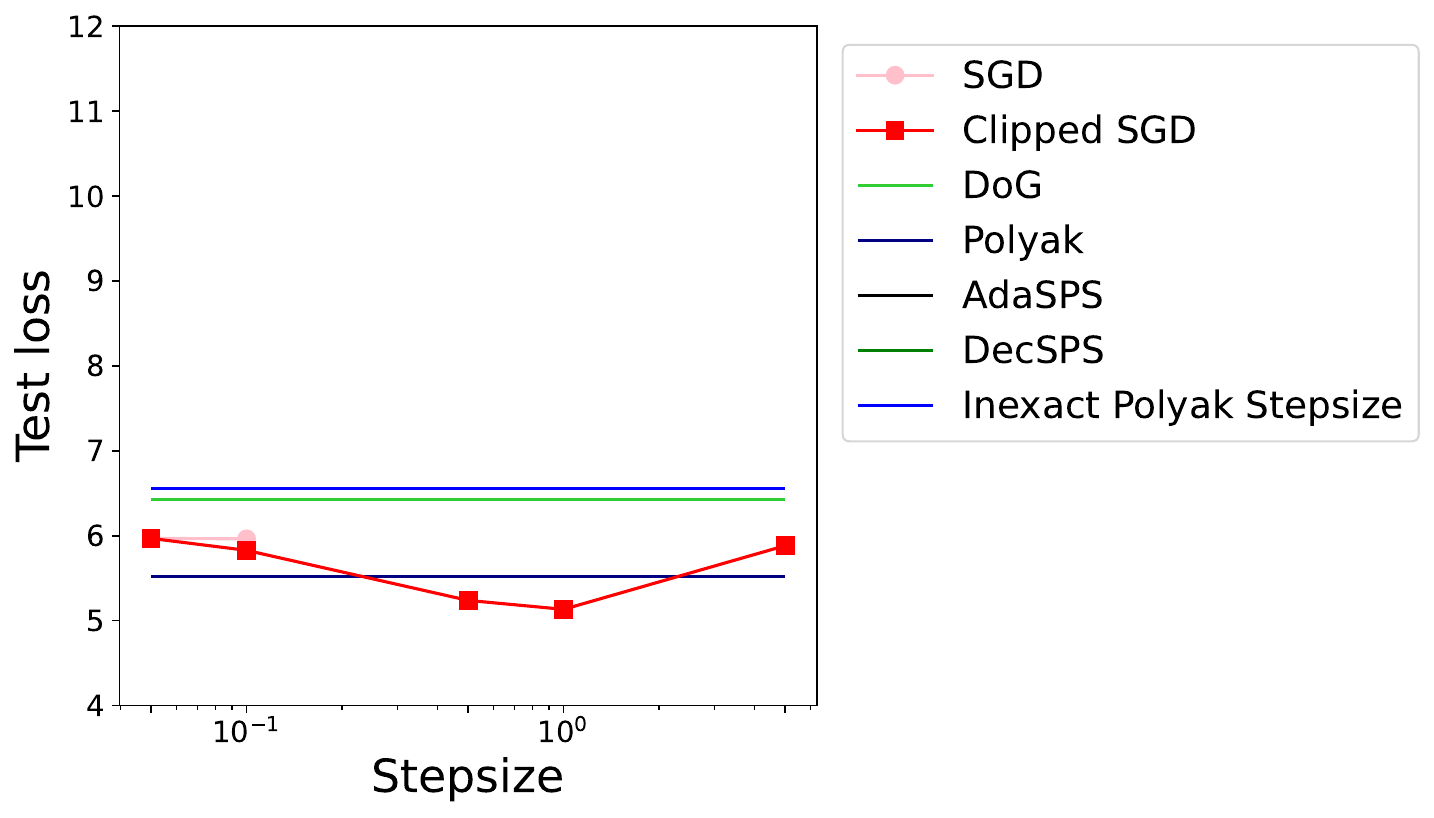}
\caption{T5}
\end{subfigure}
\hfill
\vskip - 0.2 in
\caption{The final test loss with various hyperparameter settings. For T5, the results of DecSPS and AdaSPS were omitted because their final test loss was much larger than the others, as shown in Fig.~\ref{fig:lstm}. Furthermore, the results of SGD were also omitted when the final test loss became nan or infinity.}
\label{fig:lr} 
\vskip - 0.2 in
\end{figure}
\subsection{Neural networks}
\label{section:neural_networks}
\paragraph{Setting:}
Next, we evaluated Inexact Polyak Stepsize using LSTM, Nano-GPT\footnote{\url{https://github.com/karpathy/nanoGPT}}, and T5 \citep{nawrot2023nanoT5}.
For LSTM, Nano-GPT, and T5, we used the Penn Treebank, Shakespeare, and C4 as training datasets, respectively.
For SGD and Clipped SGD, we tuned the stepsize and gradient clipping threshold on validation datasets. 
For Polyak stepsize, we showed the results when we set $f^\star$ to zero.
For Inexact Polyak Stepsize, Theorem \ref{theorem:main} requires the selection of the best parameters.
However, we do not need to choose this for neural networks because the parameters only reach the stationary point and do not reach the global minima.
See Sec.~\ref{sec:hyperparameter} for the detailed training configuration.
For all experiments, we repeated with three different seed values and reported the average.

\paragraph{Results:}
Figure \ref{fig:lstm} shows the loss curves, and Fig.~\ref{fig:lr} shows the final test losses for various hyperparameters.
The results indicate that Inexact Polyak Stepsize consistently outperform DecSPS and AdaSPS for all neural network architectures.
Although DoG performed the best for LSTM among the parameter-free methods, the training behavior of DoG was very unstable for Nano-GPT, and the loss values were much higher than those of the other methods. 
Similar to DoG, Polyak stepsize outperformed all parameter-free methods for T5, but the loss values of Polyak stepsize diverged for LSTM and Nano-GPT. 
Thus, Inexact Polyak Stepsize can consistently succeed in training models for all neural network architectures.

\begin{figure}[t!]
\centering
\hfill
\centering
\vskip - 0.2 in
\begin{subfigure}[b]{\textwidth}
    \includegraphics[width=0.37\textwidth]{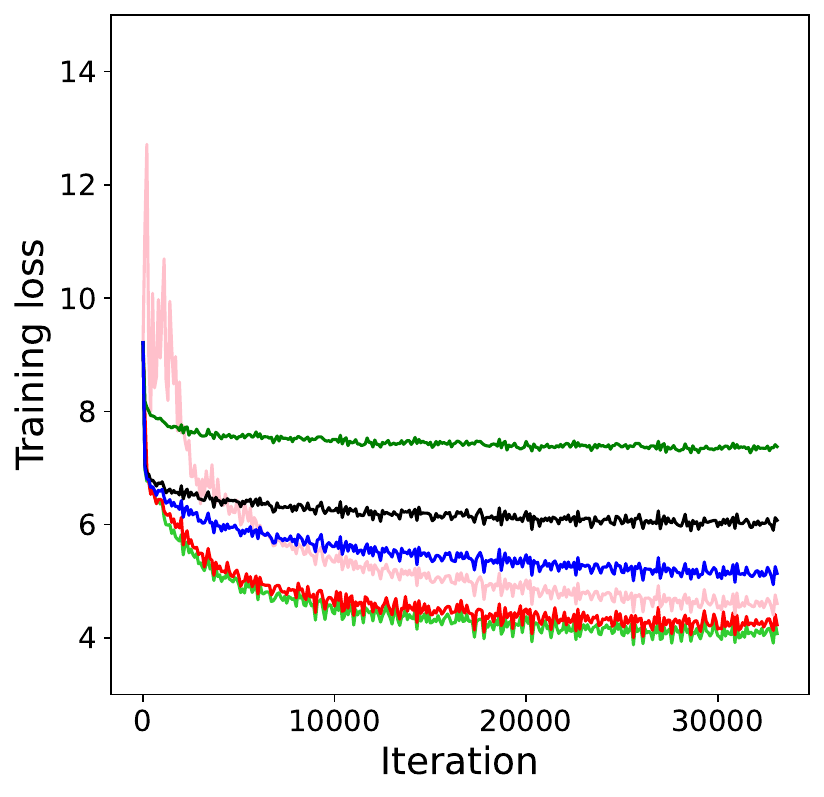}
    \hfill
    \includegraphics[width=0.61\textwidth]{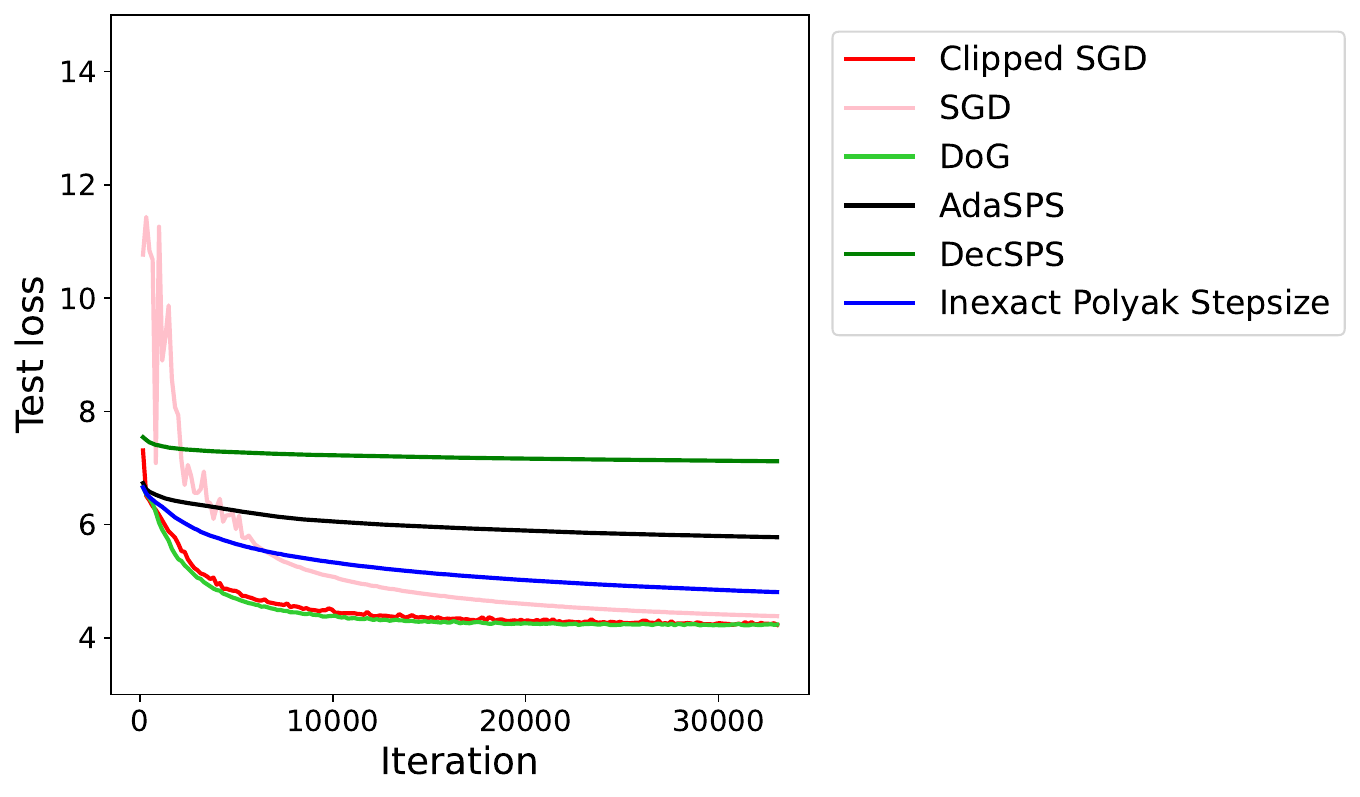}
    \vskip - 0.15 in
\caption{LSTM}
\vskip - 0.2 in
\end{subfigure}
\begin{subfigure}[b]{\textwidth}
    \includegraphics[width=0.37\textwidth]{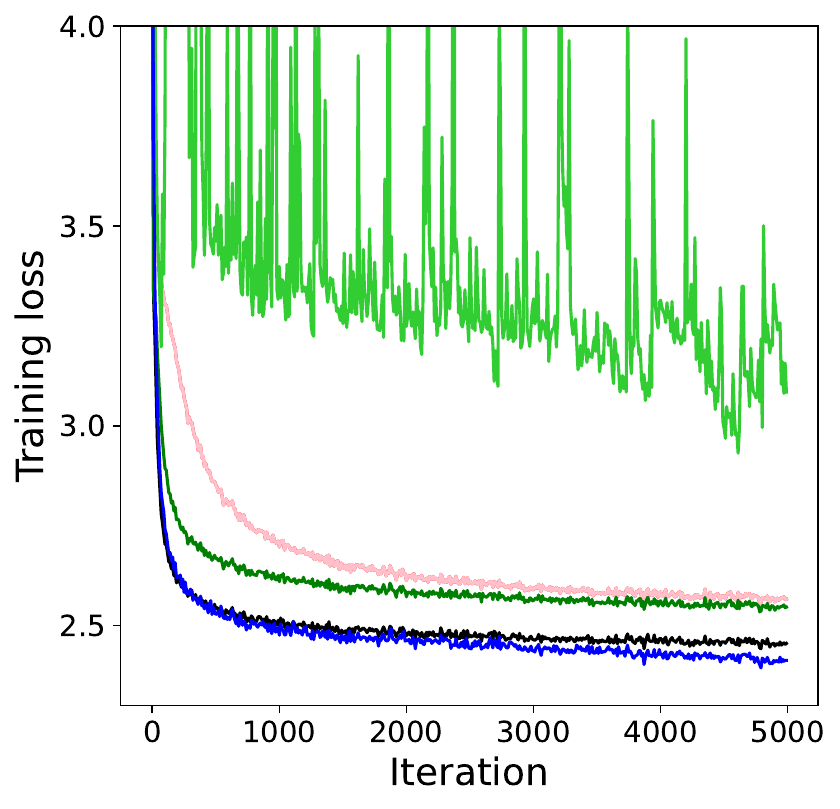}
    \hfill
    \includegraphics[width=0.61\textwidth]{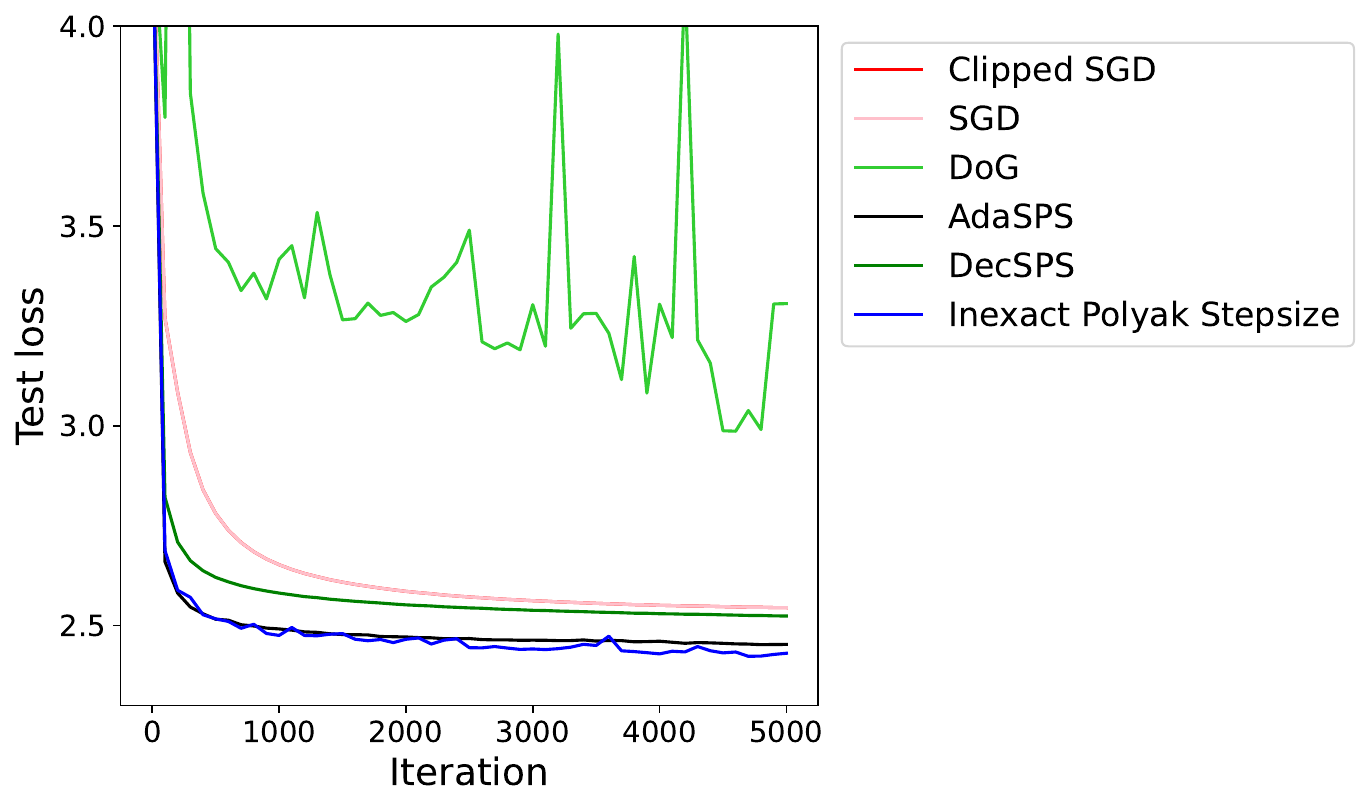}
    \vskip - 0.1 in
\caption{Nano-GPT}
\vskip - 0.2 in
\end{subfigure}
\begin{subfigure}[b]{\textwidth}
    \includegraphics[width=0.37\textwidth]{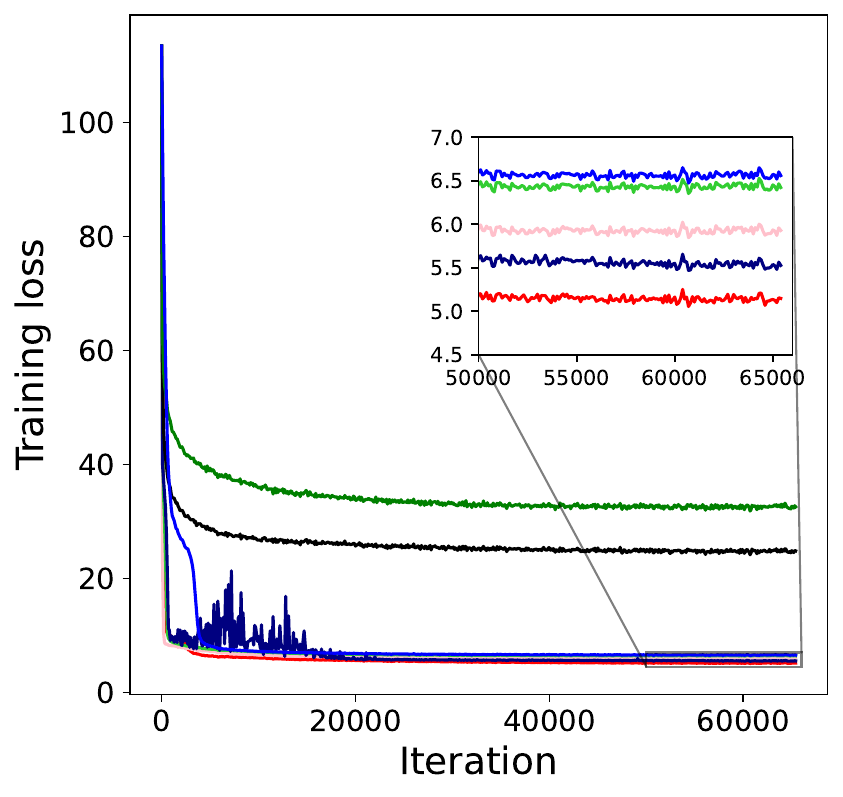}
    \hfill
    \includegraphics[width=0.61\textwidth]{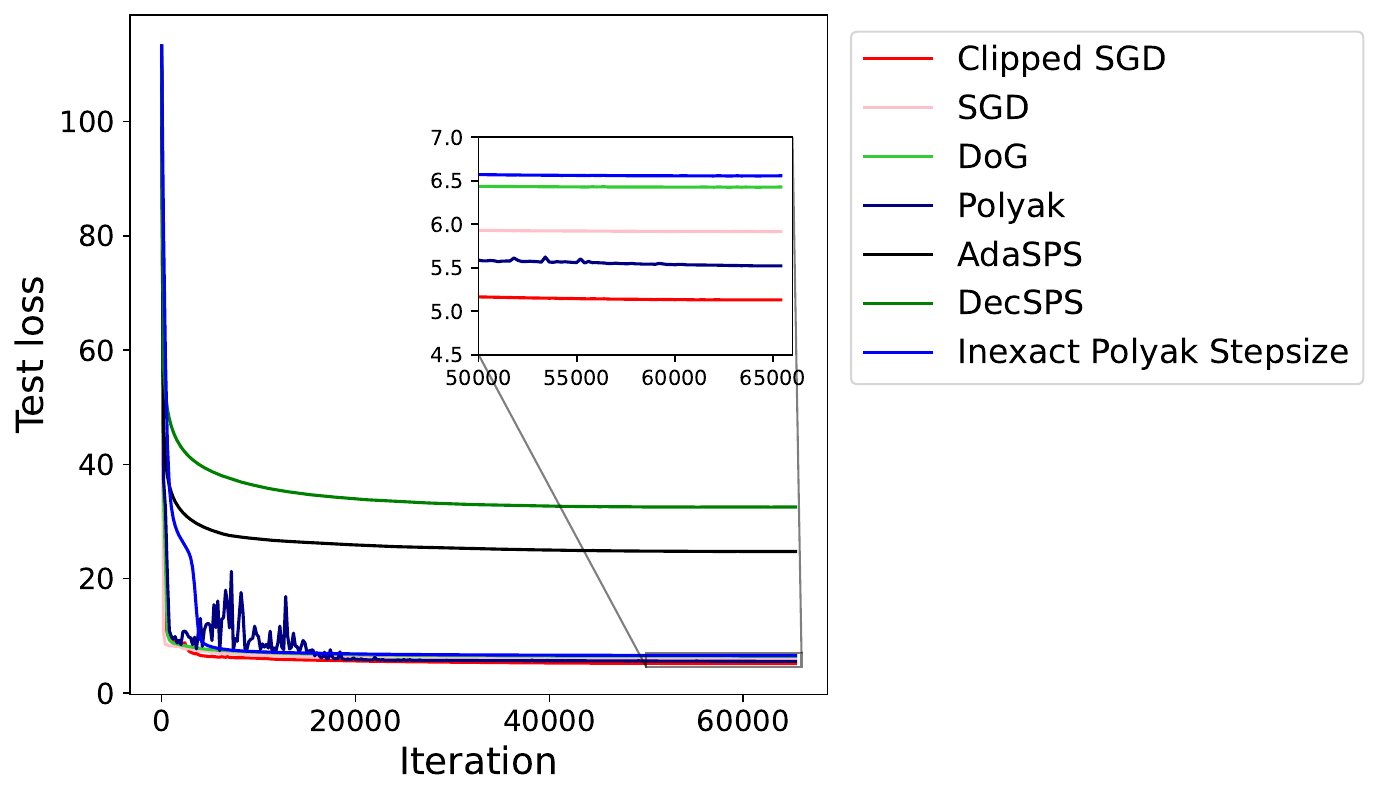}
    \vskip - 0.15 in
\caption{T5}
\end{subfigure}
\vskip - 0.01 in
\caption{Loss curves for LSTM, Nano-GPT, and T5. We plotted the training loss per $100$, $10$, and $10$ iterations for LSTM, Nano-GPT, and T5, respectively. We plotted the test loss per one epoch, $100$ iterations, and $200$ iterations, respectively. For LSTM and Nano-GPT, we found that Polyak stepsize does not converge, and its loss was much larger than that of other comparison methods. Thus, to make the figure easier to read, we omit the results of Polyak stepsize and provide the complete results, including Polyak stepsize in Sec.~\ref{sec:additional_experiments}.}
\label{fig:lstm}      
\vskip - 0.1 in
\end{figure}

\section{Conclusion}
In this study, we proposed Inexact Polyak Stepsize, which converges to the optimal solution without hyperparameter tuning at the convergence rate that is asymptotically independent of $L$ under $(L_0, L_1)$-smoothness.
Specifically, we first provided the novel convergence rate of Polyak stepsize under $(L_0, L_1)$-smoothness, revealing that Polyak stepsize can achieve exactly the same convergence rate as clipped gradient descent.
Although Polyak stepsize can improve the convergence under $(L_0, L_1)$-smoothness, Polyak stepsize requires the minimum loss value, which is a problem-specific parameter.
Then, we proposed Inexact Polyak Stepsize, which removes the problem-specific parameter from Polyak stepsize without losing the property of asymptotic independence of $L$ under $(L_0, L_1)$-smoothness.
We numerically validated our convergence results and demonstrated the effectiveness of Inexact Polyak Stepsize.

\section*{Acknowledgement}
Y.T.~was supported by KAKENHI Grant Number 23KJ1336. H.B.~and M.Y.~were supported by MEXT KAKENHI Grant Number 24K03004.
We thank Satoki Ishikawa for his helpful comments on our experiments.

\bibliography{ref}

\newpage
\newpage
\appendix

\section{Proof of Theorem \ref{theorem:main}}
\label{sec:proof_of_main}
\begin{lemma}
\label{lemma:smooth}
If Assumption \ref{assumption:smooth} holds, the following holds for any $\vx \in \mathbb{R}^d$:
\begin{align}
\label{eq:smooth2}
    \frac{1}{2 L} \| \nabla f (\vx) \|^2 \leq f (\vx) - f (\vx^\star).
\end{align}
\begin{proof}
See Lemma 2.28 in \citep{garrigos2023handbook}.
\end{proof}
\end{lemma}

\begin{lemma}
\label{lemma:generalized_smooth}
If Assumption \ref{assumption:generalized_smooth} holds, the following holds for any $\vx \in \mathbb{R}^d$:
\begin{align}
\label{eq:generalized_smooth2}
    \frac{1}{2 (L_0 + L_1 \| \nabla f (\vx) \|)} \| \nabla f (\vx) \|^2 \leq f (\vx) - f (\vx^\star).
\end{align}
\begin{proof}
See Lemma A.2 in \citep{koloskova2023revisiting}.
\end{proof}
\end{lemma}

\begin{lemma}
Assume that $f$ is convex and Assumption \ref{assumption:smooth} and \ref{assumption:generalized_smooth} hold. Let $T$ be the number of iterations and define $\tau \coloneqq \argmin_{0\leq t \leq T-1} f(\vx_t)$. Then, gradient descent with Polyak stepsize Eq.~\eqref{eq:polyak} satisfies:
\begin{align*}
    f(\vx_\tau) - f(\vx^\star) 
    &\leq \frac{8 L_0 \| \vx_0 - \vx^\star \|^2}{T} + \frac{64 L L_1^2 \| \vx_0 - \vx^\star \|^4}{T^2}.
\end{align*}
\end{lemma}
\begin{proof}
We have
\begin{align*}
    \| \vx_{t+1} - \vx^\star \|^2 
    &= \| \vx_{t} - \vx^\star \|^2 - 2 \eta_t \langle \nabla f (\vx_t), \vx_t - \vx^\star \rangle + \eta_t^2 \| \nabla f (\vx) \|^2 \\
    &\leq \| \vx_{t} - \vx^\star \|^2 - 2 \eta_t ( f (\vx_t) - f (\vx^\star)) + \eta_t^2 \| \nabla f (\vx) \|^2,
\end{align*}
where we use the convexity of $f$ in the inequality.

\textbf{Case when $\| \nabla f (\vx_t) \| \leq \tfrac{L_0}{L_1}$:}
Substituting the stepsize, we get
\begin{align*}
    \| \vx_{t+1} - \vx^\star \|^2 
    &\leq \| \vx_{t} - \vx^\star \|^2 - \frac{f (\vx_t) - f (\vx^\star)}{\| \nabla f (\vx_t) \|^2} ( f (\vx_t) - f (\vx^\star)) \\
    &\leq \| \vx_{t} - \vx^\star \|^2 - \frac{1}{2 (L_0 + L_1 \| \nabla f (\vx_t) \|)} ( f (\vx_t) - f (\vx^\star)) \\
    &\leq \| \vx_{t} - \vx^\star \|^2 - \frac{1}{4 L_0} ( f (\vx_t) - f (\vx^\star)),
\end{align*}
where we use Lemma \ref{lemma:generalized_smooth} in the second inequality.
Unrolling the above inequality, we obtain
\begin{align*}
    f (\vx_t) - f (\vx^\star)
    &\leq 4 L_0 \left( \| \vx_{t} - \vx^\star \|^2 - \| \vx_{t+1} - \vx^\star \|^2 \right).
\end{align*}

\textbf{Case when $\| \nabla f (\vx_t) \| > \tfrac{L_0}{L_1}$:}
Substituting the stepsize, we get
\begin{align*}
    \| \vx_{t+1} - \vx^\star \|^2 
    &\leq \| \vx_{t} - \vx^\star \|^2 - \frac{(f(\vx_t) - f(\vx^\star))^2}{\| \nabla f(\vx_t)\|^2} \\
    &\leq \| \vx_{t} - \vx^\star \|^2 - \sqrt{\frac{f(\vx_t) - f(\vx^\star)}{2L}} \frac{f(\vx_t) - f(\vx^\star)}{\| \nabla f(\vx_t)\|},
\end{align*}
where we use Lemmas \ref{lemma:smooth} in the last inequality.
Then, $\| \nabla f (\vx_t) \| > \tfrac{L_0}{L_1}$ implies
\begin{align*}
    \frac{L_0 + L_1 \| \nabla f (\vx_t)\|}{2 L_1 \| \nabla f (\vx_t) \|} < 1.
\end{align*}
Thus, we get
\begin{align*}
    \| \vx_{t+1} - \vx^\star \|^2 
    &\leq \| \vx_{t} - \vx^\star \|^2 - \sqrt{\frac{f(\vx_t) - f(\vx^\star)}{2L}} \frac{f(\vx_t) - f(\vx^\star)}{\| \nabla f(\vx_t)\|^2} \frac{L_0 + L_1 \| \nabla f (\vx_t)\|}{2 L_1} \\
    &\leq \| \vx_{t} - \vx^\star \|^2 - \frac{1}{4 L_1} \sqrt{\frac{f(\vx_t) - f(\vx^\star)}{2L}},
\end{align*}
where we use Lemma \ref{lemma:generalized_smooth} in the last inequality.
Unrolling the above inequality and multiplying $L_0$ on both sides, we get
\begin{align*}
    \frac{L_0}{L_1} \sqrt{ \frac{f (\vx) - f (\vx^\star)}{2 L}} 
    &\leq 4 L_0 \left( \| \vx_{t} - \vx^\star \|^2 - \| \vx_{t+1} - \vx^\star \|^2 \right).
\end{align*}

\textbf{Summing the two cases:}
Define $\mathcal{T}_1$ and $\mathcal{T}_2$ as follows:
\begin{align*}
    \mathcal{T}_1 \coloneqq \left\{ t \middle| \| \nabla f (\vx_t) \| \leq \frac{L_0}{L_1} \right\}, \;\; \mathcal{T}_2 \coloneqq \left\{ t \middle| \| \nabla f (\vx_t) \| > \frac{L_0}{L_1} \right\}.
\end{align*}
We obtain
\begin{align*}
    \sum_{t \in \mathcal{T}_1} \left( f(\vx_t) - f(\vx^\star) \right) + \frac{L_0}{L_1} \sum_{t \in \mathcal{T}_2} \sqrt{ \frac{f (\vx_t) - f (\vx^\star)}{2 L}} 
    &\leq 4 L_0 \| \vx_0 - \vx^\star \|^2.
\end{align*}
Then, the above inequality implies
\begin{align*}
    \frac{1}{T} \sum_{t \in \mathcal{T}_1} f(\vx_t) - f(\vx^\star) 
    &\leq \frac{4 L_0 \| \vx_0 - \vx^\star \|^2}{T}, \\
    \frac{1}{T} \sum_{t \in \mathcal{T}_2} \sqrt{ f (\vx_t) - f (\vx^\star)} 
    &\leq \frac{4 L_1 \sqrt{2L} \| \vx_0 - \vx^\star \|^2}{T}.
\end{align*}
Using $a^2 \geq 2 a b - b^2$, we obtain for any $b \in \mathbb{R}$
\begin{align*}
    \frac{1}{T} \sum_{t \in \mathcal{T}_1} \left( 2 b \sqrt{f(\vx_t) - f(\vx^\star)} - b^2 \right) 
    &\leq \frac{4 L_0 \| \vx_0 - \vx^\star \|^2}{T}.
\end{align*}
Thus, when $b > 0$, we obtain
\begin{align*}
    \frac{1}{T} \sum_{t \in \mathcal{T}_1} \sqrt{f(\vx_t) - f(\vx^\star)} 
    &\leq \frac{4 L_0 \| \vx_0 - \vx^\star \|^2}{2 b T} + \frac{b}{2}.
\end{align*}
Choosing $b = \sqrt{\tfrac{4 L_0 \| \vx_0 - \vx^\star \|^2}{T}}$, we get
\begin{align*}
    \frac{1}{T} \sum_{t \in \mathcal{T}_1} \sqrt{ f(\vx_t) - f(\vx^\star) }  
    &\leq \sqrt{ \frac{4 L_0 \| \vx_0 - \vx^\star \|^2}{T}}.
\end{align*}
Thus, we get
\begin{align*}
    \frac{1}{T} \sum_{t=0}^{T-1} \sqrt{ f(\vx_t) - f(\vx^\star) }  
    &\leq \sqrt{ \frac{4 L_0 \| \vx_0 - \vx^\star \|^2}{T}} + \frac{4 L_1 \sqrt{2L} \| \vx_0 - \vx^\star \|^2}{T}.
\end{align*}
Defining $\tau \coloneqq \argmin_t f (\vx_t)$, we get
\begin{align*}
    \sqrt{ f(\vx_\tau) - f(\vx^\star) }  
    &\leq \sqrt{ \frac{4 L_0 \| \vx_0 - \vx^\star \|^2}{T}} + \frac{4 L_1 \sqrt{2L} \| \vx_0 - \vx^\star \|^2}{T}.
\end{align*}
Squaring the both sides, and using $(a + b)^2 \leq 2 a^2 + 2 b^2$ for all $a, b \in \mathbb{R}$, we obtain 
\begin{align*}
    f(\vx_\tau) - f(\vx^\star) 
    &\leq \frac{8 L_0 \| \vx_0 - \vx^\star \|^2}{T} + \frac{64 L L_1^2 \| \vx_0 - \vx^\star \|^4}{T^2}.
\end{align*}
This concludes the statement.
\end{proof}

\newpage
\section{Proof of Theorem \ref{theorem:main_lower_bound}}
\label{sec:proof_of_lower_bound}

\begin{lemma}
\label{lemma:inexact_polyak}
Assume that $f$ is convex and Assumptions \ref{assumption:smooth} and \ref{assumption:generalized_smooth} hold. Let $T$ be the number of iterations and define $\tau \coloneqq \argmin_{0\leq t \leq T-1} f(\vx_t)$. If $f(\vx_t) - f^\star \geq \tfrac{\sigma^2}{\sqrt{T}}$ for all $t$, then gradient descent with stepsize Eq.~\eqref{eq:inexact_polyak} satisfies:
\begin{align*}
    f(\vx_\tau) - f^\star
    &\leq \frac{8 L_0 \| \vx_0 - \vx^\star \|^2 + 2 \sigma^2}{\sqrt{T}}
    + \frac{128 L_1^2 L \| \vx_0 - \vx^\star \|^4}{T} 
    + \frac{8 L_1^2 \sigma^4 L}{L_0^2 T}.
\end{align*}
where $\vx^\star \coloneqq \argmin_{\vx} f (\vx)$ and $\sigma^2 \coloneqq f^\star - l^\star$.
\end{lemma}
\begin{proof}
By the convexity of $f$, we have
\begin{align*}
    \| \vx_{t+1} - \vx^\star \|^2 
    &= \| \vx_{t} - \vx^\star \|^2 - 2 \eta_t \langle \nabla f (\vx_t), \vx_t - \vx^\star \rangle + \eta_t^2 \| \nabla f (\vx_t) \|^2 \\
    &\leq \| \vx_{t} - \vx^\star \|^2 - 2 \eta_t ( f (\vx_t) - f^\star) + \eta_t^2 \| \nabla f (\vx_t) \|^2.
\end{align*}
Substituting the stepsize Eq.~\eqref{eq:inexact_polyak}, we get
\begin{align}
    \| \vx_{t+1} - \vx^\star \|^2 
    &\leq \| \vx_{t} - \vx^\star \|^2 - 2 \eta_t ( f (\vx_t) - f^\star) + \frac{\eta_t}{\sqrt{T}} (f (\vx_t) - l^\star) \nonumber \\
    &\leq \| \vx_{t} - \vx^\star \|^2 - \eta_t (2 - \frac{1}{\sqrt{T}}) ( f (\vx_t) - f^\star) + \frac{\eta_t \sigma^2}{\sqrt{T}} \nonumber \\
    \label{eq:upper_bound_of_distance}
    &\leq \| \vx_{t} - \vx^\star \|^2 - \eta_t ( f (\vx_t) - f^\star) + \frac{\eta_t \sigma^2}{\sqrt{T}},
\end{align}
where we use $T \geq 1$ in the last inequality.
Unrolling the above inequality and dividing by $\eta_t$, we obtain
\begin{align} 
\label{eq:upper_bound_of_loss}
    f (\vx_t) - f^\star
    &\leq \frac{\| \vx_{t} - \vx^\star \|^2 - \| \vx_{t+1} - \vx^\star \|^2}{\eta_t}  + \frac{\sigma^2}{\sqrt{T}}.
\end{align}

\textbf{Case when $\| \nabla f (\vx_t) \| \leq \tfrac{L_0}{L_1}$:}
From $f(\vx_t) - f^\star \geq \tfrac{\sigma^2}{\sqrt{T}}$ and Eq.~\eqref{eq:upper_bound_of_distance}, we obtain
\begin{align}
\label{eq:distance}
    \| \vx_{t} - \vx^\star \|^2 - \| \vx_{t+1} - \vx^\star \|^2 \geq 0.
\end{align}
Thus, we get
\begin{align*} 
    f (\vx_t) - f^\star
    &\leq 2 (L_0 + L_1 \| \nabla f (\vx_t) \|) \sqrt{T} (\| \vx_{t} - \vx^\star \|^2 - \| \vx_{t+1} - \vx^\star \|^2)  + \frac{\sigma^2}{\sqrt{T}} \\
    &\leq 4 L_0 \sqrt{T} (\| \vx_{t} - \vx^\star \|^2 - \| \vx_{t+1} - \vx^\star \|^2)  + \frac{\sigma^2}{\sqrt{T}},
\end{align*}
where we use $f^\star \geq l^\star$ and Lemma \ref{lemma:generalized_smooth} for the first inequality and use $\| \nabla f (\vx_t) \| \leq \tfrac{L_0}{L_1}$ for the last inequality.

\textbf{Case when $\| \nabla f (\vx_t) \| > \tfrac{L_0}{L_1}$:}
From Lemma \ref{lemma:generalized_smooth}, we have
\begin{align*}
    \eta_t 
    \geq \frac{f (\vx_t) - f^\star}{\sqrt{T} \| \nabla f(\vx_t) \|^2} 
    \geq \frac{1}{2 (L_0 + L_1  \| \nabla f(\vx_t) \|) \sqrt{T}}.
\end{align*}
Then, we obtain
\begin{align*}
    \eta_t 
    \geq \frac{1}{4 L_1  \| \nabla f(\vx_t) \| \sqrt{T}}
    \geq \frac{1}{4 L_1 \sqrt{2 L T (f (\vx_t) - f^\star)}},
\end{align*}
where we use $\| \nabla f (\vx_t) \| > \tfrac{L_0}{L_1}$ for the first inequality,
and Lemma \ref{lemma:smooth} for the last inequality.
Combining Eqs.~\eqref{eq:upper_bound_of_loss} and \eqref{eq:distance}, we obtain
\begin{align*} 
    f (\vx_t) - f^\star
    &\leq 4 L_1 \sqrt{2 L T (f (\vx_t) - f^\star)} (\| \vx_{t} - \vx^\star \|^2 - \| \vx_{t+1} - \vx^\star \|^2)  + \frac{\sigma^2}{\sqrt{T}}.
\end{align*}
Furthermore, from $\| \nabla f (\vx_t) \| > \tfrac{L_0}{L_1}$ and Lemma \ref{lemma:smooth}, we obtain
\begin{align*}
    \sqrt{f (\vx_t) - f^\star} 
    \geq \frac{L_0}{L_1} \sqrt{\frac{f (\vx_t) - f^\star}{\| \nabla f (\vx_t) \|^2}}
    \geq \frac{L_0}{L_1} \sqrt{\frac{1}{2 L}}.
\end{align*}
Thus, we get
\begin{align*} 
    &f (\vx_t) - f^\star \\
    &\leq 4 L_1 \sqrt{2 L T (f (\vx_t) - f^\star)} (\| \vx_{t} - \vx^\star \|^2 - \| \vx_{t+1} - \vx^\star \|^2)  + \frac{L_1 \sigma^2}{L_0 \sqrt{T}} \sqrt{2 L (f (\vx_t) - f^\star)}.
\end{align*}
Dividing by $\tfrac{L_1 \sqrt{2 L (f (\vx_t) - f^\star)}}{L_0}$, we get
\begin{align*} 
    \frac{L_0}{L_1}\sqrt{\frac{f (\vx_t) - f^\star}{2L}} 
    &\leq 4 L_0 \sqrt{T} (\| \vx_{t} - \vx^\star \|^2 - \| \vx_{t+1} - \vx^\star \|^2)  + \frac{\sigma^2}{\sqrt{T}}.
\end{align*}
\textbf{Summing the two cases:}
Define $\mathcal{T}_1$ and $\mathcal{T}_2$ as follows:
\begin{align*}
    \mathcal{T}_1 \coloneqq \left\{ t \middle| \| \nabla f (\vx_t) \| \leq \frac{L_0}{L_1} \right\}, \;\; \mathcal{T}_2 \coloneqq \left\{ t \middle| \| \nabla f (\vx_t) \| > \frac{L_0}{L_1} \right\}.
\end{align*}
We obtain
\begin{align*}
    \frac{1}{T} \left( \sum_{t \in \mathcal{T}_1} \left( f(\vx_t) - f^\star \right) + \frac{L_0}{L_1} \sum_{t \in \mathcal{T}_2} \sqrt{ \frac{f (\vx) - f^\star}{2 L}} \right)
    &\leq \frac{4 L_0 \| \vx_0 - \vx^\star \|^2 + \sigma^2}{\sqrt{T}}.
\end{align*}
The above inequality implies that
\begin{align*}
    \frac{1}{T} \sum_{t \in \mathcal{T}_1} \left( f(\vx_t) - f^\star \right)
    &\leq \frac{4 L_0 \| \vx_0 - \vx^\star \|^2 + \sigma^2}{\sqrt{T}}, \\
    \frac{1}{T} \sum_{t \in \mathcal{T}_2} \sqrt{ f (\vx) - f^\star} 
    &\leq \frac{4 L_1 \sqrt{2 L} \| \vx_0 - \vx^\star \|^2}{\sqrt{T}} +  \frac{L_1 \sigma^2 \sqrt{2 L}}{L_0 \sqrt{T}}.
\end{align*}
Using $a^2 \geq 2 a b - b^2$, we obtain for any $b \in \mathbb{R}$
\begin{align*}
    \frac{1}{T} \sum_{t \in \mathcal{T}_1} \left( 2 b \sqrt{f(\vx_t) - f^\star} - b^2 \right)
    &\leq \frac{4 L_0 \| \vx_0 - \vx^\star \|^2 + \sigma^2}{\sqrt{T}}.
\end{align*}
Thus, when $b>0$, we obtain
\begin{align*}
    \frac{1}{T} \sum_{t \in \mathcal{T}_1} \sqrt{f(\vx_t) - f^\star}
    &\leq \frac{4 L_0 \| \vx_0 - \vx^\star \|^2 + \sigma^2}{2 b \sqrt{T}} + \frac{b}{2}.
\end{align*}
Choosing $b = \sqrt{\frac{4 L_0 \| \vx_0 - \vx^\star \|^2 + \sigma^2}{\sqrt{T}}}$, we get
\begin{align*}
    \frac{1}{T} \sum_{t \in \mathcal{T}_1} \sqrt{f(\vx_t) - f^\star}
    &\leq \sqrt{\frac{4 L_0 \| \vx_0 - \vx^\star \|^2 + \sigma^2}{\sqrt{T}}}.
\end{align*}
Thus, we get
\begin{align*}
    \frac{1}{T} \sum_{t=0}^{T-1} \sqrt{f(\vx_t) - f^\star}
    &\leq \sqrt{\frac{4 L_0 \| \vx_0 - \vx^\star \|^2 + \sigma^2}{\sqrt{T}}}
    + \frac{4 L_1 \sqrt{2 L} \| \vx_0 - \vx^\star \|^2}{\sqrt{T}} 
    + \frac{L_1 \sigma^2 \sqrt{2 L}}{L_0 \sqrt{T}}.
\end{align*}
Defining $\tau \coloneqq \argmin_t f (\vx_t)$, we get
\begin{align*}
    \sqrt{f(\vx_\tau) - f^\star}
    &\leq \sqrt{\frac{4 L_0 \| \vx_0 - \vx^\star \|^2 + \sigma^2}{\sqrt{T}}}
    + \frac{4 L_1 \sqrt{2 L} \| \vx_0 - \vx^\star \|^2}{\sqrt{T}} 
    + \frac{L_1 \sigma^2 \sqrt{2 L}}{L_0 \sqrt{T}}.
\end{align*}
Squaring the both sides, and using $(a + b)^2 \leq 2 a^2 + 2 b^2$ for all $a, b \in \mathbb{R}$, we obtain 
\begin{align*}
    f(\vx_\tau) - f^\star
    &\leq \frac{8 L_0 \| \vx_0 - \vx^\star \|^2 + 2 \sigma^2}{\sqrt{T}}
    + \frac{128 L_1^2 L \| \vx_0 - \vx^\star \|^4}{T} 
    + \frac{8 L_1^2 \sigma^4 L}{L_0^2 T}.
\end{align*}
This concludes the statement.
\end{proof}

\begin{lemma}
Assume that $f$ is convex and Assumptions \ref{assumption:smooth} and \ref{assumption:generalized_smooth} hold. Let $T$ be the number of iterations and define $\tau \coloneqq \argmin_{0\leq t \leq T-1} f(\vx_t)$. Then, gradient descent with stepsize Eq.~\eqref{eq:inexact_polyak} satisfies:
\begin{equation}
    f(\vx_{\tau}) - f (\vx^\star) \leq \mathcal{O} \left( \frac{ L_0 \| \vx_0 - \vx^\star\|^2 + \sigma^2}{\sqrt{T}} + \frac{L L_1^2 \| \vx_0 - \vx^\star \|^4}{T} + \frac{L_1^2 L \sigma^4}{L_0^2 T} \right),
\end{equation}
where $\vx^\star \coloneqq \argmin_{\vx} f (\vx)$ and $\sigma^2 \coloneqq f^\star - l^\star$.
\end{lemma}
\begin{proof}
If there exists $t$ such that $f(\vx_t) - f^\star < \tfrac{\sigma^2}{\sqrt{T}}$, we have
\begin{align*}
    f (\vx_\tau) - f^\star \leq f (\vx_t) - f^\star < \frac{\sigma^2}{\sqrt{T}}.
\end{align*}
Then, if $f(\vx_t) - f^\star \geq \tfrac{\sigma^2}{\sqrt{T}}$ for all $t$, Lemma \ref{lemma:inexact_polyak} shows that
\begin{align*}
    f(\vx_\tau) - f^\star
    &\leq \frac{8 L_0 \| \vx_0 - \vx^\star \|^2 + 2 \sigma^2}{\sqrt{T}}
    + \frac{128 L_1^2 L \| \vx_0 - \vx^\star \|^4}{T} 
    + \frac{8 L_1^2 \sigma^4 L}{L_0^2 T}.
\end{align*}
By combining the above two cases, we have the desired statement.
\end{proof}

\section{Additional theoretical result}

\begin{lemma}
\label{lemma:twice_defferential}
Let $f$ be a function such that $\| \nabla^2 f (\vx) \| \leq L_0 + L_1 \| \nabla f (\vx) \|$ holds for any $\vx$.
For any $\vx, \vy$ such that $\| \vx - \vy \| \leq \tfrac{1}{L_0}$, we have
\begin{align*}
    \| \nabla f (\vx) - \nabla f(\vy) \| \leq 2 (L_0 + L_1  \| \nabla f (\vx) \|) \| \vx -\vy \|.
\end{align*}
\end{lemma}
\begin{proof}
See Lemma A.2 in \citep{zhang2020improved}.    
\end{proof}

\begin{proposition}
\label{prop:generalized_smooth}
For any $L_0 \geq 0$ and $L_1 \geq 0$, $f(x) \coloneqq \tfrac{L_0 L_1^2}{72} x^4 + \tfrac{L_0}{4} x^2$ is $(L_0, L_1)$-smooth.
\end{proposition}
\begin{proof}
Since $f (x)$ is twice differentiable, we have
\begin{align*}
    | \nabla^2 f (x) |
    &= \frac{L_0 L_1^2}{6} x^2 + \frac{L_0}{2}.
\end{align*}
Using $\tfrac{L_1}{6} x^2 + \tfrac{3}{2 L_1} \geq |x|$,  we obtain
\begin{align*}
    | \nabla^2 f (x) |
    &\leq \frac{L_0 L_1^2}{6} \left(\frac{L_1}{6} x^2 + \frac{3}{2 L_1} \right) |x| + L_0 \\
    &= \frac{L_1}{2} \left| \frac{L_0 L_1^2}{18} x^3 + \frac{L_0}{2} x \right| + \frac{L_0}{2} \\
    &= \frac{L_1}{2} | \nabla f (x) | + \frac{L_0}{2}.    
\end{align*}
From Lemma \ref{lemma:twice_defferential}, we have the desired statement.
\end{proof}

\newpage
\section{Hyperparameter settings}
\label{sec:hyperparameter}

\subsection{Synthetic function}
In our experiments, we ran the clipped gradient descent with the following hyperparameters and tuned the hyperparameters by grid search.
\begin{table}[h!]
\centering
\vskip - 0.2 in 
\caption{Hyperparameter settings for clipped gradient descent.}
\begin{tabular}{lc}
\toprule
Learning Rate               & $\{ 1, 1.0 \times 10^{-1}, \cdots, 1.0 \times 10^{-8} \}$ \\
Gradient Clipping Threshold & $ \{ 0.01, 0.1, 1, 5, 10, 15, 20, \infty \}$ \\
\bottomrule
\end{tabular}
\vskip - 0.1 in 
\end{table}

\begin{table}[h!]
\centering
\vskip - 0.2 in 
\caption{Hyperparameters selected by grid search.}
\label{table:best_hypara_synthetic}
\begin{tabular}{lcccc}
\toprule
 & \textbf{Gradient Descent} & & \multicolumn{2}{c}{\textbf{Clipped Gradient Descent}} \\
 & Learning Rate & & Learning Rate & Gradient Clipping Threshold \\
\midrule
$L_1=1$     & $1.0 \times 10^{-1}$ & & $0.1$ & $20$ \\
$L_1=10$    & $1.0 \times 10^{-3}$ & & $0.1$ & $10$ \\
$L_1=100$   & $1.0 \times 10^{-5}$ & & $0.1$ & $10$ \\
$L_1=1000$  & $1.0 \times 10^{-7}$ & & $0.1$ & $10$ \\
\bottomrule
\end{tabular}
\vskip - 0.1 in 
\end{table}

\subsection{Neural networks}

In our experiments, we used the following training configuration:
\begin{itemize}
    \item \textbf{LSTM:} \url{https://github.com/salesforce/awd-lstm-lm}
    \item \textbf{Nano-GPT:} \url{https://github.com/karpathy/nanoGPT}
    \item \textbf{T5:} \url{https://github.com/PiotrNawrot/nanoT5}
\end{itemize}
We ran all experiments on an A100 GPU.
For Clipped SGD and SGD, we tuned the stepsize and gradient clipping threshold using the grid search.
See Tables \ref{table:hypara_lstm}, \ref{table:hypara_nanogpt}, and \ref{table:hypara_t5} for detailed hyperparameter settings, and see Table \ref{table:best_hypara} for the selected hyperparameters.

\begin{table}[h!]
\centering
\vskip - 0.1 in 
\caption{Hyperparameter settings for LSTM.}
\label{table:hypara_lstm}
\begin{tabular}{lc}
\toprule
Learning Rate               &  $\{ 100, 50, 10, 1, 0.1, 0.01\}$\\
Gradient Clipping Threshold &  $\{ 0.5, 1, \cdots, 4.5, 5, \infty\}$\\
Batch Size & 80 \\
\bottomrule
\end{tabular}
\vskip - 0.1 in 
\end{table}

\begin{table}[h!]
\centering
\vskip - 0.1 in 
\caption{Hyperparameter settings for Nano-GPT.}
\label{table:hypara_nanogpt}
\begin{tabular}{lc}
\toprule
Learning Rate               &  $\{ 1, 0.5, 0.1, \cdots, 0.0005, 0.0001 \}$\\
Gradient Clipping Threshold &  $\{ 1, 2, \cdots, 9, 10, \infty\}$ \\
Batch Size & 64 \\
\bottomrule
\end{tabular}
\vskip - 0.1 in 
\end{table}

\begin{table}[h!]
\centering
\vskip - 0.1 in 
\caption{Hyperparameter settings for T5.}
\label{table:hypara_t5}
\begin{tabular}{lc}
\toprule
Learning Rate               &  $\{ 5.0, 1.0, 0.5, 0.1, 0.05 \}$\\
Gradient Clipping Threshold &  $\{ 1, 2, 3, \infty\}$\\
Batch Size & 128 \\
\bottomrule
\end{tabular}
\vskip - 0.1 in 
\end{table}

\begin{table}[h!]
\centering
\vskip - 0.1 in 
\caption{Hyperparameters selected by grid search. Three values correspond to the selected hyperparameters for different seed values.}
\label{table:best_hypara}
\begin{tabular}{lcccc}
\toprule
 & \textbf{Gradient Descent} & & \multicolumn{2}{c}{\textbf{Clipped Gradient Descent}} \\
 & Learning Rate & & Learning Rate & Gradient Clipping Threshold \\
\midrule
LSTM     & $10$ / $10$ / $10$ & & $10$ / $50$ / $50$ & $0.5$ / $1$ / $0.5$ \\
Nano-GPT & $0.001$ / $0.001$ / $0.001$ & & $0.001$ / $0.001$ / $0.001$ & $\infty$ / $10$ / $10$ \\
T5       & $0.1$ / $0.1$ / $0.05$ & & $1$ / $1$ / $1$ & $2$ / $2$ / $2$ \\
\bottomrule
\end{tabular}
\end{table}

\newpage
\section{Additional numerical evaluation}
\label{sec:additional_experiments}

\begin{figure}[h]
\centering
\hfill
\centering
\vskip - 0.1 in
\begin{subfigure}[b]{\textwidth}
    \includegraphics[width=0.37\textwidth]{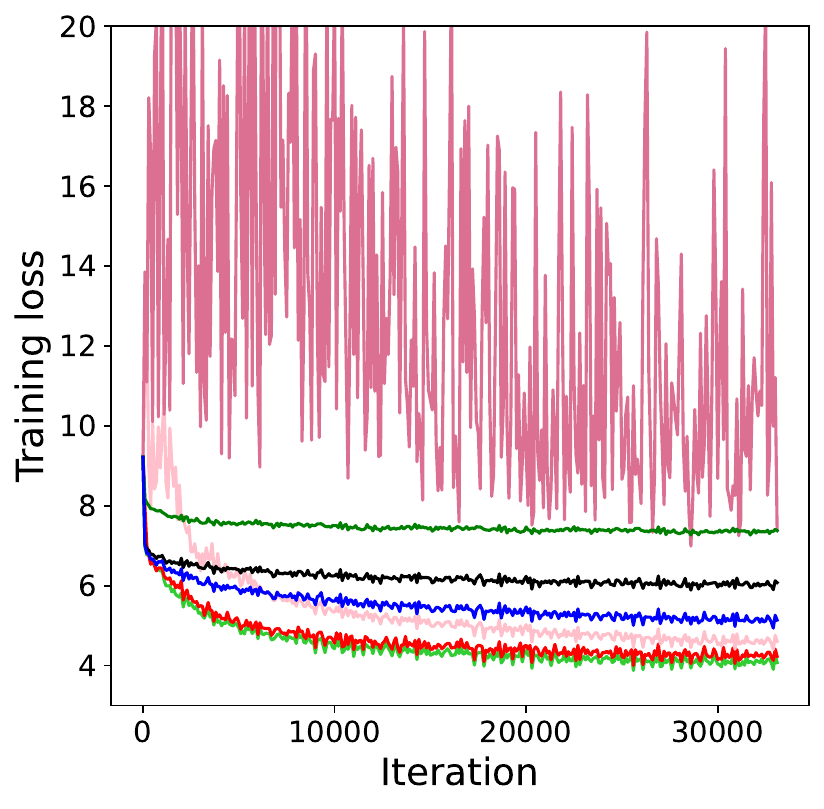}
    \hfill
    \includegraphics[width=0.61\textwidth]{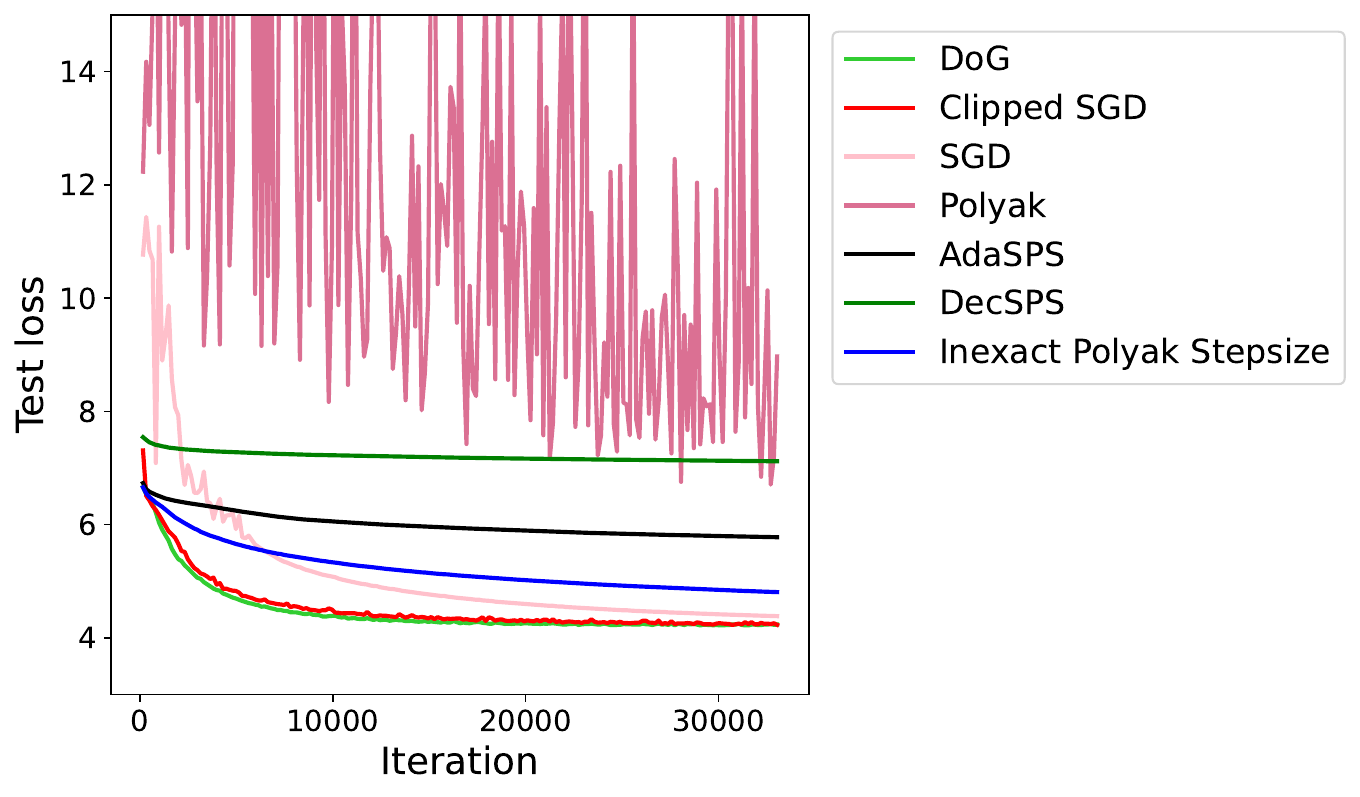}
    \vskip - 0.1 in
\caption{LSTM}
\vskip - 0.2 in
\end{subfigure}
\begin{subfigure}[b]{\textwidth}
    \includegraphics[width=0.37\textwidth]{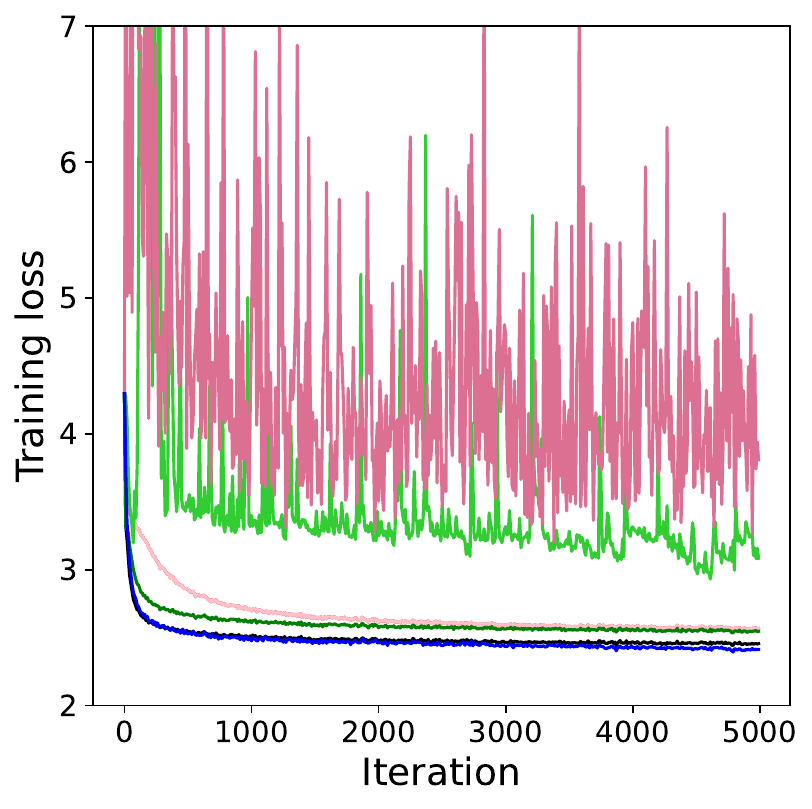}
    \hfill
    \includegraphics[width=0.61\textwidth]{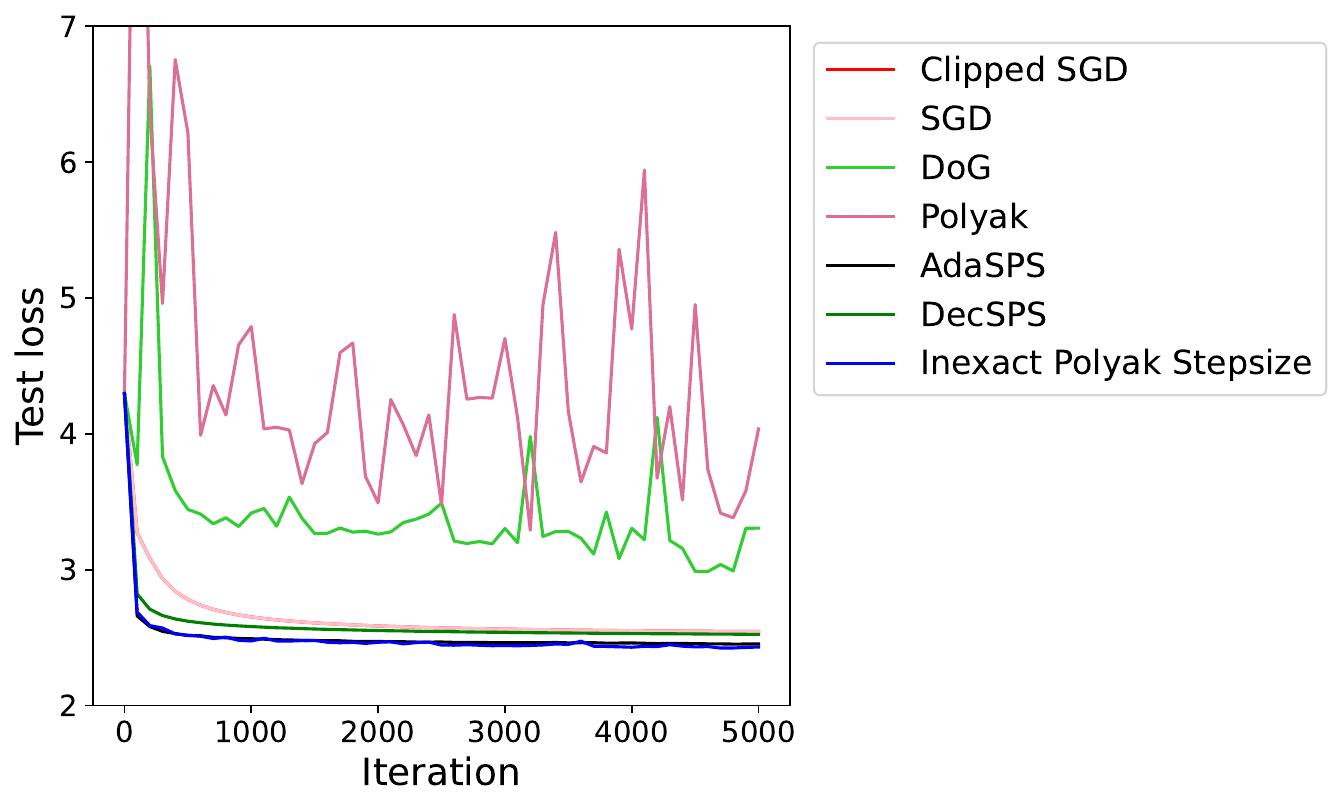}
    \vskip - 0.1 in
\caption{Nano-GPT}
\vskip - 0.2 in
\end{subfigure}
\caption{Loss curves for LSTM and Nano-GPT. We plotted the training loss per $100$, $10$, and $10$ iterations for LSTM, Nano-GPT, and T5, respectively. We plotted the test loss per one epoch, $100$ iterations, and $200$ iterations, respectively.}
\label{fig:lstm}      
\end{figure}

\begin{figure}[h]
\centering
\hfill
\centering
\vskip - 0.1 in
\begin{subfigure}[b]{\textwidth}
    \includegraphics[width=0.37\textwidth]{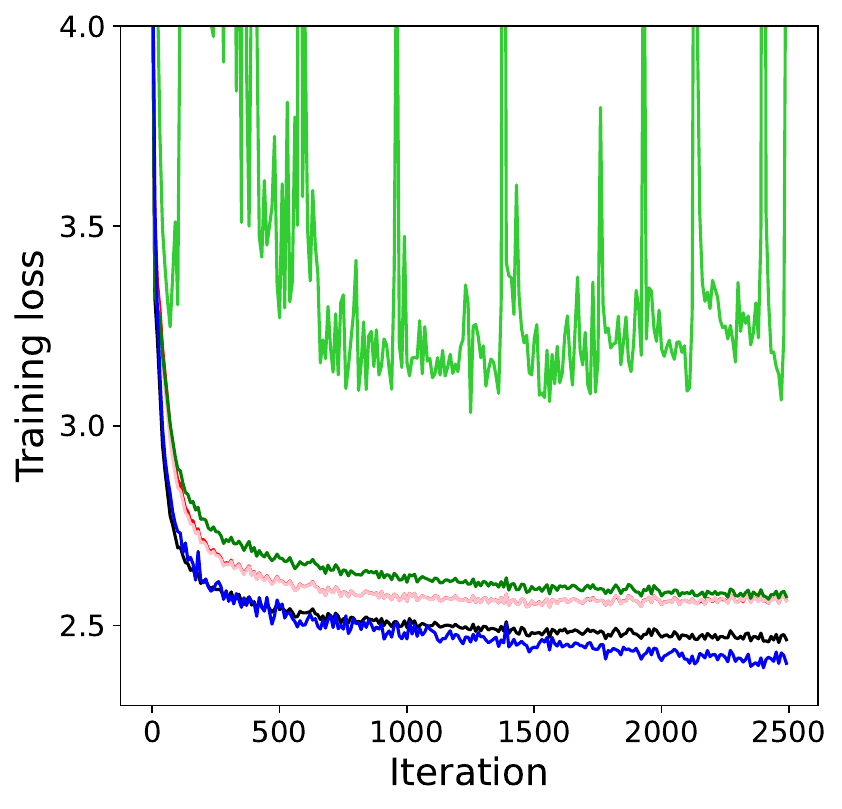}
    \hfill
    \includegraphics[width=0.61\textwidth]{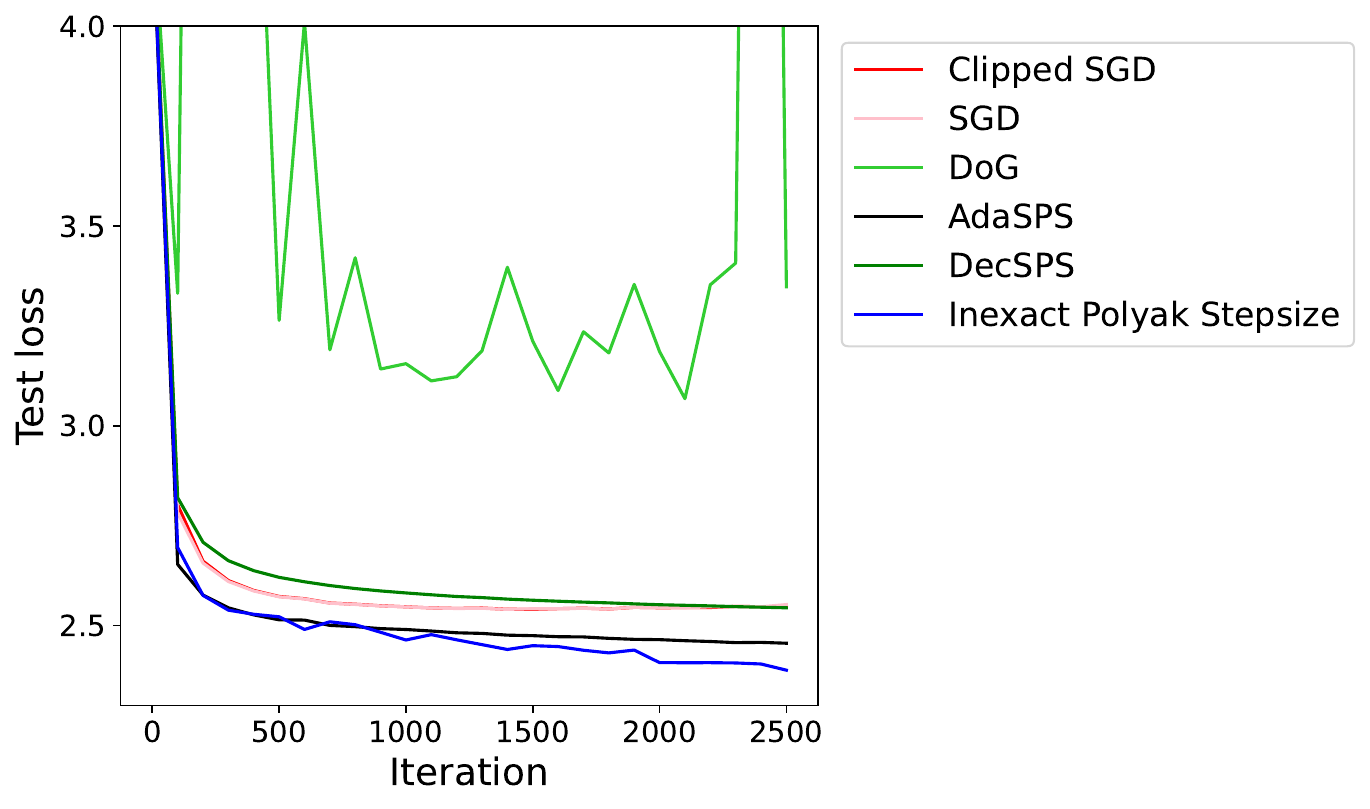}
    \vskip - 0.1 in
\caption{$T=2500$}
\vskip - 0.2 in
\end{subfigure}
\begin{subfigure}[b]{\textwidth}
    \includegraphics[width=0.37\textwidth]{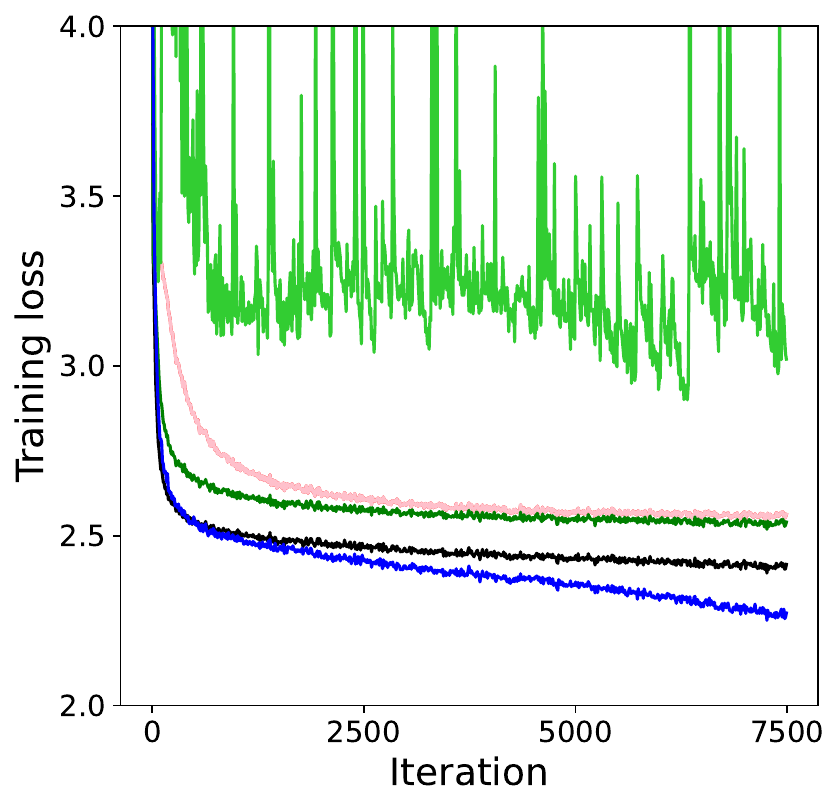}
    \hfill
    \includegraphics[width=0.61\textwidth]{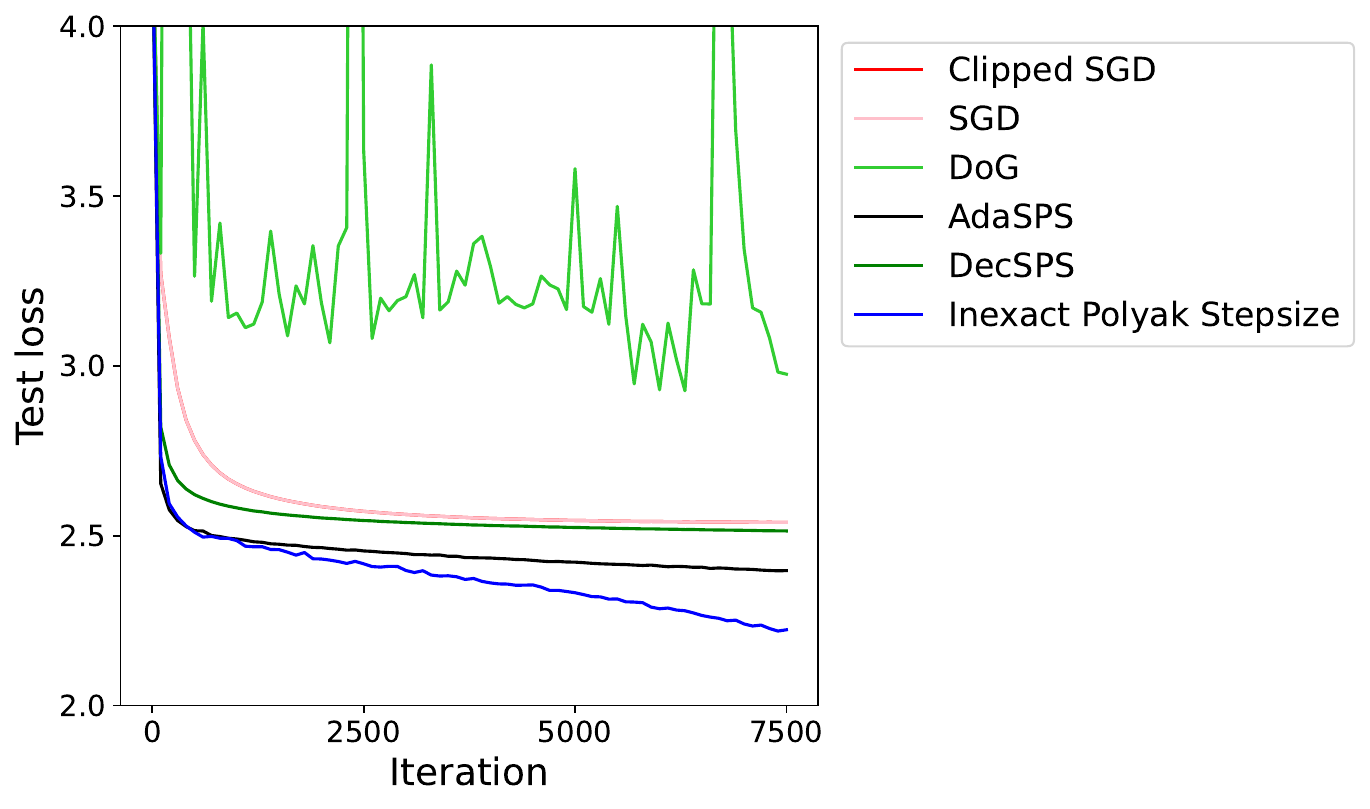}
    \vskip - 0.1 in
\caption{$T=7500$}
\vskip - 0.2 in
\end{subfigure}
\caption{Loss curves for Nano-GPT with different $T$.}
\end{figure}

\end{document}